\documentclass[11pt]{article}
\usepackage{amsfonts,latexsym,amstext}
\usepackage{amsmath}
\usepackage{comment}
\usepackage{placeins}
\usepackage{subcaption}
\usepackage{amssymb}
\usepackage[english]{babel}
\usepackage[utf8]{inputenc}
\usepackage{bbm}
\usepackage{mathtools}
\usepackage{etoolbox}

\DeclareMathOperator{\VC}{VC}
\DeclareMathOperator{\fat}{fat}

\usepackage{amsthm}
\usepackage{graphicx,geometry}
\usepackage{natbib}

\usepackage[usenames]{xcolor}
\definecolor{red}{rgb}{1.0,0.0,0.0}

\definecolor{blu}{rgb}{0.0,0.0,1.0}

\definecolor{gre}{rgb}{0.03,0.50,0.03}

\usepackage[markup=colored,commandnameprefix=always]{changes}
\definechangesauthor[name={Kaveh}, color=blu]{KS}
\definechangesauthor[name={Martin}, color=red]{MA}

\usepackage{hyperref}

\newcommand{\vvvert}{\vert\vert\vert}

\DeclarePairedDelimiterX{\normiii}[1]{\vvvert}{\vvvert}
{\ifblank{#1}{\:\cdot\:}{#1}}

\theoremstyle{plain}

\newtheorem{theorem}{Theorem}[section]
\newtheorem{lemma}[theorem]{Lemma}
\newtheorem{proposition}[theorem]{Proposition}
\newtheorem{definition}[theorem]{Definition}

\newtheorem{remark}[theorem]{Remark}
\newtheorem{corollary}[theorem]{Corollary}
\numberwithin{equation}{section}
\def\sqr#1#2{{\vcenter{\vbox{\hrule height.#2pt \hbox{\vrule
 width.#2pt height#1pt \kern#1pt \vrule
width.#2pt} \hrule height.#2pt}}}}

\def\ds{\begin{displaystyle}}
\def\eds{\end{displaystyle}}

\def\<{\left\langle }
\def\>{\right\rangle }

\def\to{\rightarrow}

\newcommand{\softmax}{\operatorname{softmax}}

\begin{document}

\global\long\def\S{\mathrm{\mathcal{S}}}%

\title{\bfseries Generalised Eigenvalue Geometry of Semantic Adversarial
Attacks\footnote{%
\setlength{\parindent}{0pt}%
\textsuperscript{1} Data Science Institute, London School of Economics and
Political Science.\par
\textsuperscript{2} Department of Mathematics, London School of Economics and
Political Science.\par
\textsuperscript{3} The Inclusion Initiative, London School of Economics and
Political Science.%
}}

\author{
Martin Anthony$^{1,2}$ \and
Kaveh Salehzadeh Nobari$^{1,3}$
}

\date{}

\maketitle
\begin{abstract}
Recent empirical work shows that semantically equivalent paraphrases can systematically fool financial sentiment classifiers: the paraphrased input remains close to the original under a strong reference embedding model, yet shifts the target model's representation far enough to flip the predicted class.
Existing theoretical accounts of adversarial robustness are either restricted to single-model threat models or remain at the level of empirical algorithms.
This paper develops a continuous local model of semantic paraphrase perturbations that captures the two-model structure, and shows that the worst-case local displacement of the target representation under a proxy budget is governed by the top generalised eigenvalue of a matrix pencil $(A, B)$ formed from the Jacobians of the two embedders.
The resulting attackability index $\lambda^*(x)$ is intrinsic to the chosen local paraphrase geometry and the two embedding maps, yields a closed-form prediction-flip condition for a fixed affine readout, and leads to conservative population and finite-sample attackability certificates.
For uniform control across classes of affine readouts, we prove a distribution-free VC bound for the binary attackability indicator class and a scale-sensitive margin bound for an attackability-adjusted margin that subtracts a local geometric penalty from the ordinary classifier margin.
A separate section bridges the continuous theory to the discrete paraphrase searches used in practice, identifying an asymmetry between success and failure of finite search and giving a covering condition under which the two settings agree.
Finally, we outline an empirical verification strategy based on soft-token relaxations of one-hot token representations and finite sets of generated paraphrases, showing how the local eigenvalue geometry, prediction-flip condition, and finite-search approximation can be assessed on a deployed financial-text classifier.
\end{abstract}

%\end{document}

%\textbf{Key words}:
%
%\bigskip \noindent
%
%\textbf{AMS classification}:
%
%%93E20 (Optimal stochastic control),
%%60H20 (Stochastic integral equations),
%%47D07 (Markov semigroups and applications to diffusion processes),
%%49L20 (Dynamic programming method),
%%35R15 (Partial differential equations on infinite-dimensional spaces).
%%34K50 (Stochastic functional-differential equations)
%%93C23 (Systems governed by functional-differential equations)
%
%\bigskip \noindent
%
%\textbf{Acknowledgements}:

\section{Introduction}\label{sec:Intro}

Adversarial attacks on machine learning models have been studied extensively in both computer vision and natural language processing. A large body of work has shown that small input perturbations can cause otherwise accurate classifiers to produce confidently wrong predictions, demonstrated empirically in computer vision \citep{szegedy2014intriguing, goodfellow2015explaining} and analysed theoretically through the lens of robust optimisation \citep{madry2017towards, tsipras2018robustness}. In natural language processing, the analogue is more delicate: text is discrete, and meaningful perturbations cannot be defined by an $\ell_p$ norm. Adversarial examples are instead constructed by word substitutions or paraphrases that preserve meaning while altering a classifier's prediction \citep{alzantot2018generating, jin2020bert}. Recent empirical work in financial NLP has shown that this vulnerability is systematic: semantically equivalent paraphrases can reliably fool deployed sentiment classifiers while remaining close to the original input under a strong reference embedding model \citep{TURETKEN2026107698}.

The empirical evidence reveals a structural asymmetry. The paraphrase preserves meaning, as judged by a strong reference (or \emph{proxy}) embedding model, yet displaces the representation produced by the deployed (or \emph{target}) classifier far enough to cross the decision boundary. The adversary, in effect, exploits a disagreement between two models about how perturbations affect meaning. This paper asks what theoretical structure governs that disagreement.

We take a different route from the prevailing empirical and algorithmic literature on adversarial attacks. Rather than proposing a new attack method, we study the local geometry of semantic paraphrase perturbations under a two-embedding threat model, in which a proxy embedding defines the adversary's semantic budget, say $\eta$, and a target embedding determines the classifier's response. Modelling paraphrases through a continuous local parameter $u$ -- motivated by the continuous relaxations that underpin gradient-based text attacks such as the GBDA framework of \citet{guo2021gradient}, which in turn build on the Gumbel--softmax reparameterisation of \citet{jang2016categorical} -- we show that the worst-case local displacement of the target representation under a proxy budget admits a closed-form characterisation as a generalised Rayleigh quotient. Its optimal value is  the leading generalised eigenvalue of a matrix pencil $(A, B)$ formed from the Jacobians of the two embedders at $x$. We refer to this leading eigenvalue $\lambda^*(x)$ as the \emph{local attackability index}: a scalar diagnostic, intrinsic to the input and the two embedders, that quantifies the adversary's leverage at $x$.

From this characterisation, we derive a closed-form  condition for prediction flips   in the \emph{linearised local} problem, and we show that the worst-case classifier-direction displacement is  dominated by $\lambda^*(x)$. We then lift the analysis to the population level  for a fixed affine readout. Defining an \emph{attackability-adjusted margin}   $Z_w(x):= \gamma_w(x) / \sqrt{\Sigma_w(x)}$, where $\gamma_w(x)$  is the geometric margin of the   readout  at $x$ and $\Sigma_w(x)$ measures the squared target displacement along the   readout  direction per unit proxy budget, we characterise the population attackability   $\mathcal{A}_w(\eta):= \mathbb{P}[Z_w(x) < \eta]$ as a left-limit distribution function of $Z_w$. A finite-sample concentration result follows from the  Dvoretzky--Kiefer--Wolfowitz   inequality. We also give uniform versions over data-dependent affine readout classes, using VC, fat-shattering, and Rademacher-complexity arguments in the spirit of classical margin theory, while keeping the new geometric ingredient in the adjusted margin.

The closest body of theoretical work to ours concerns the transferability of adversarial attacks across models. \citet{papernot2016transferability} and \citet{demontis2019adversarial} examine when attacks crafted against a surrogate model remain effective against a different target, and \citet{tramer2017ensemble} studies ensemble adversarial training as a defence. Our setting is structurally different. We do not ask whether a single perturbation transfers across two trained classifiers, but rather how the local geometries of two embedding maps interact to determine semantic vulnerability under a shared paraphrase budget. Where \citet{demontis2019adversarial} characterise transferability through a first-order cosine-alignment criterion between surrogate and target gradients,
our local attackability index is characterised through the spectrum of the matrix pencil $(A,B)$, with the leading generalised eigenvalue playing the role of a worst-case displacement quantity and the full spectrum recording relative anisotropy between the target and proxy geometries.

Our use of local geometry is distinct from manifold-based decompositions of adversarial risk. \citet{zhang2022manifold} assume that data lie on a smooth manifold embedded in the ambient input space and decompose adversarial risk into tangential and normal components. By contrast, we keep the text space discrete and introduce only a local continuous relaxation of paraphrase directions. The central object in our analysis is therefore not the tangent--normal splitting of a data manifold, but the relative pullback geometry of two embedding maps, captured by the generalised eigenvalue spectrum of the pencil $(A,B)$.

Finally, we add an empirical verification exercise in the same financial-text setting that motivates the paper. We use the Financial PhraseBank of \citet{malo2014good} because it provides labelled economic and financial sentences on which financial sentiment classifiers are commonly evaluated. We take FinBERT \citep{araci2019finbert} as the target model, treating its final classification head as the fixed affine readout in our theory, and use Sentence-BERT \citep{reimers2019sentence} as the proxy embedding model that defines semantic closeness. The exercise assesses whether the local continuous perturbation geometry, the readout-specific eigenvalue flip condition, and the finite-search approximation predict observed vulnerability to semantically constrained paraphrases. Since the FinBERT readout is fixed rather than selected from the evaluation sample, the empirical analysis is tied primarily to the fixed-readout attackability and adjusted-margin quantities; the VC, fat-shattering and Rademacher bounds provide the corresponding uniform extensions for data-dependent readout classes.

The contributions of the paper are as follows. We introduce a continuous local model of semantic paraphrase perturbations and a closed-form characterisation of worst-case local displacement as a generalised Rayleigh quotient (Section~\ref{sec:Local Semantic}). We establish a linearised margin condition for prediction flips under semantic constraints, together with a conservative worst-direction no-flip certificate (Section~\ref{sec:Prediction-Flip}). We lift the analysis to the population level for a fixed affine readout, giving a margin-tail attackability bound and a finite-sample concentration result for the empirical attackability curve (Section~\ref{sec:population}). We then give uniform versions over data-dependent affine readout and margin classes (Sections~\ref{sec:hypotheses} and~\ref{sec:margins}), including both dimension-dependent covering bounds and a trace-sensitive Rademacher bound, and relate the idealised local adversary to finite paraphrase search through a covering-radius condition on the generated candidates (Section~\ref{sec:Finite Paraphrase}). Finally, we provide an empirical verification exercise in financial sentiment classification, using a fixed FinBERT affine readout, a Sentence-BERT proxy embedding, and labelled Financial PhraseBank sentences to assess the local eigenvalue geometry, the readout-specific flip condition, and the finite-search approximation (Section~\ref{sec:experiments}). Section~\ref{sec:Setup} introduces the setup, and Section~\ref{sec:Discussion} concludes.

\section{Setup}\label{sec:Setup}

Let $\mathcal{X}$ denote the space of text samples. We keep $\mathcal X$ discrete throughout. The continuous objects introduced below are not coordinates on $\mathcal X$ itself, but local relaxations of paraphrase directions around a fixed base text. We consider the outputs of two embedding models,
\begin{equation}
 e_M: \mathcal{X} \to \mathbb{R}^{d_M},\qquad
 e_P: \mathcal{X} \to S^{d_P-1},
\end{equation}
where $e_M$ is the \emph{target} model under attack, $e_P$ is a
\emph{proxy} model used by the adversary to measure semantic
similarity, and $S^{k-1} = \{v \in \mathbb{R}^k: \|v\|_2 = 1\}$
is the unit hypersphere in $\mathbb{R}^k$ \footnote{The target is left unnormalised so that the deployed affine head acts on it exactly; the proxy is normalised so that the proxy distance has the cosine reading used in
Section~\ref{sec:experiments}.}. The embedding dimensions
$d_M$ and $d_P$ need not coincide; the local matrices $A=J_M^\top J_M$
and $B=J_P^\top J_P$ are $q\times q$ regardless, and the readout will
act on $\mathbb{R}^{d_M}$.

Given $x \in \mathcal{X}$, the adversary seeks a paraphrase $x'$
that preserves semantic meaning while maximally displacing the
target representation $e_M(x)$. Semantic similarity is
operationalised through the proxy distance
\[
 d_P(x, x'):= \|e_P(x') - e_P(x)\|_2,
\]
and the semantic neighbourhood of radius $\eta > 0$ is defined as
\begin{equation}\label{eq:constraint}
 \mathcal{N}_\eta(x):= \{x' \in \mathcal{X}: d_P(x, x') \leq \eta\}.
\end{equation}
The adversarial representation problem is then
\begin{equation}\label{eq:adv-rep}
 \sup_{x' \in \mathcal{N}_\eta(x)} \|e_M(x') - e_M(x)\|_2.
\end{equation}
In other words, among all paraphrases judged semantically close to $x$
by the proxy $P$, how far can the target $M$'s internal
representation be displaced? This formulation can be viewed as a
theoretical abstraction of the empirical adversarial attack proposed by
\citet{TURETKEN2026107698}.

We further assume a binary affine readout on top of the target
representation: a weight vector $w \in \mathbb{R}^d_M$ and a bias
$b \in \mathbb{R}$ such that the predicted class on input $x$ is
\begin{equation}\label{eq:classifier}
 \hat{y}(x):= \operatorname{sign}\left(w^\top e_M(x) + b\right)
\end{equation} 
where $\operatorname{sign}\colon\mathbb R\to\{-1,+1\}$ is the sign function,
with the convention $\operatorname{sign}(0):=+1$. In particular, $\hat y(x)\in\{-1,+1\}$,
and we will work with this label convention throughout.
We write
\[
s_{w,b}(x)=w^\top e_M(x)+b
\]
for the affine readout or score, so that
\[
\hat{y}_{w,b}(x)=\operatorname{sign}(s_{w,b}(x))
\]
is the induced binary classifier. The word ``affine'' refers only to the final readout; the embedding
map $e_M$ may be an arbitrary nonlinear representation, such as a
transformer embedding. For the pointwise geometric statements in
Sections~\ref{sec:Prediction-Flip} and~\ref{sec:population}, we normalise
$w$ to have $\|w\|_2 = 1$, since rescaling
$(w,b)\mapsto(w/\|w\|_2,b/\|w\|_2)$ preserves the classifier
\eqref{eq:classifier}. The decision hyperplane is then the level set
$\{v \in \mathbb{R}^d_M: w^\top v + b = 0\}$, and the geometric margin of
$x$ is the perpendicular distance from $e_M(x)$ to it,
 \begin{equation}\label{eq:margin}
 \gamma _w (x):= \left| w^\top e_M(x) + b \right|,
 \end{equation}
 where the dependence on $b$ is suppressed in the notation. In the later
uniform margin bounds of Sections~\ref{sec:hypotheses}
and~\ref{sec:margins}, the scale of $w$ is kept explicit, because margin
bounds depend on the norm constraint imposed on the readout class. There
the same expression $\gamma_w(x)=|w^\top e_M(x)+b|$ is a score margin
rather than the perpendicular geometric distance. A paraphrase
$x'$ induces a \emph{prediction flip} at $x$ if
$\hat{y}(x') \neq \hat{y}(x)$
\citep[see~Ch. 15 of][for definition of linear classifiers and margins in the context of support vector machines.]{shalev2014understanding}.  For a smooth nonlinear readout score $g(e_M(x))$, the same first-order analysis applies locally with $w$ replaced by $\nabla g(e_M(x))$ or, in the multiclass case, by the gradient of the relevant logit difference.

\subsection{Local Semantic Perturbation Model}\label{sec:Local Semantic}

Since $\mathcal{X}$ is discrete, the maps $e_M$ and $e_P$ admit
no direct notion of differentiation, and a local analysis at the
text level is unavailable. We therefore model paraphrasing through a
continuous local relaxation. The relaxation is local to the base text\footnote{Section~\ref{sec:experiments} describes empirical instantiations of this relaxation, including soft-token perturbations and finite generated paraphrase sets.}.
For notational simplicity, we take the local coordinate dimension $q$ to be
fixed across $x$. Thus, for each $x\in\mathcal X$, let
$U_x \subset \mathbb{R}^{q}$ be an open neighbourhood of the origin, and
assume the existence of two $C^2$ maps
\begin{equation}\label{eq:smooth-charts}
 E_{M,x}: U_x \to \mathbb{R}^{d_M},\qquad
 E_{P,x}: U_x \to S^{d_P-1},\qquad
 E_{M,x}(0) = e_M(x), \quad E_{P,x}(0) = e_P(x).
\end{equation}
The latent coordinate $u \in U_x$ parameterises an effective
$q$-dimensional family of local paraphrase directions at $x$, and
$E_{M,x}(u), E_{P,x}(u)$ denote the target and proxy embeddings of the
paraphrase indexed by $u$. This is not a claim that natural-language
paraphrases globally form a smooth manifold. Rather, $q$ is a chosen local
coordinate dimension for a continuous approximation to paraphrase directions
near $x$. The arguments below are pointwise and extend verbatim to
input-dependent dimensions $q_x$, provided $B(x)$ is invertible on the chosen
local coordinate space. When the base text $x$ is fixed, we suppress this
dependence and write
\[
(U,E_M,E_P)\coloneqq (U_x,E_{M,x},E_{P,x}).
\]

 \begin{lemma}[Local semantic representation]\label{lem:localrep}
For each base text $x$, define the local Jacobians
\[
 J_M(x):= \left.\frac{\partial E_{M,x}}{\partial u}\right|_{u=0},
 \qquad
 J_P(x):= \left.\frac{\partial E_{P,x}}{\partial u}\right|_{u=0},
\]
and the pullback metric matrices
\begin{equation}\label{eq:metric-matrices}
 A(x):= J_M(x)^\top J_M(x), \qquad B(x):= J_P(x)^\top J_P(x).
\end{equation}
When $x$ is fixed, write $J_M=J_M(x)$, $J_P=J_P(x)$,
$A=A(x)$, and $B=B(x)$. Then $J_M$ is a $d_M\times q$ matrix, and $J_P$ is $d_P\times q$, and a second-order Taylor expansion of the squared Euclidean
distance at $u=0$ yields
\begin{align}
 \|E_P(u) - E_P(0)\|_2^2 &= u^\top B u + o(\|u\|^2),
 \label{eq:quad-P} \\
 \|E_M(u) - E_M(0)\|_2^2 &= u^\top A u + o(\|u\|^2),
 \label{eq:quad-M}
\end{align}
where $A$ and $B$ are $q\times q$ positive semi-definite matrices, i.e., the
pullbacks of the Euclidean metric on $\mathbb{R}^d_M$ and $\mathbb{R}^{d_P} $ through $E_M$ and $E_P$,
respectively. These quadratic forms encode the local geometry:
$u^\top B u$ measures local semantic displacement as seen by the
proxy, and $u^\top A u$ measures local representation displacement
in the target.
\end{lemma}

Given the attacker's objective \eqref{eq:adv-rep} subject to the
semantic neighbourhood constraint \eqref{eq:constraint}, together
with the local representation in \eqref{eq:smooth-charts}, the
attacker's problem becomes
\begin{align}
 \sup_{u \in U}\lVert E_M(u)-E_M(0)\rVert_2^2
 \quad\text{subject to}\quad
 \lVert E_P(u)-E_P(0)\rVert_2^2 \leq \eta^2.
 \end{align}
  By  Lemma~\ref{lem:localrep}, and for sufficiently small $\eta$ so that
the relevant proxy ellipsoid remains in the chart $U$, this reduces to
the leading-order quadratic problem
 \begin{equation}\label{eq:local-problem}
 \max_{u \in \mathbb{R}^q} u^\top A u
 \qquad\text{subject to}\qquad
 u^\top B u \leq \eta^2.
\end{equation}
  Throughout the main statements we  assume $B$ is  positive definite.   This
means that the proxy embedding is locally sensitive to every retained
paraphrase coordinate. If $B$ is singular and $u^\top A u>0$ for some
$u\in\ker B$, then there is a first-order target movement at zero proxy
cost, the Rayleigh quotient $u^\top Au/u^\top Bu$ is unbounded, and the
local attackability problem is ill-posed. A finite quotient-space
analysis therefore requires $\ker B\subseteq\ker A$, or an explicit
modelling decision to remove or regularise the proxy-null directions
(for example, replacing $B$ by $B+\nu I$ for some $\nu>0$).

\begin{proposition}[Attacker's solution]\label{prop:solution}
  Let $A \succeq 0$ and $B \succ 0$. Define the top generalised eigenvalue
\begin{equation}\label{eq:lambdaAB}
 \lambda_{\max}(A,B):= \sup_{u \neq 0}\frac{u^\top A u}{u^\top B u}
 = \lambda_{\max}\left(B^{-1/2}AB^{-1/2}\right).
\end{equation}
Then the optimal value of \eqref{eq:local-problem} is
\[
 \eta^2\lambda_{\max}(A,B).
\]
The optimum is attained at any $B$-normalised generalised eigenvector
$u^*$ associated with $\lambda_{\max}(A,B)$, that is,
  \[
 A u^* = \lambda_{\max}(A,B) B u^*,
 \qquad
 u^{*\top} B u^* = \eta^2.
\]
\end{proposition}

Recall that the operator norm of a matrix $M\in\mathbb R^{d\times d}$ is
\[
\|M\|_{\mathrm{op}}:=\sup_{\|v\|_2=1}\|Mv\|_2,
\]
and that for symmetric positive semidefinite $M$ it coincides with the largest
eigenvalue $\lambda_{\max}(M)$.

\begin{remark}[Whitened-proxy interpretation]\label{rem:whitened-proxy}
Assume $B\succ 0$ and $A=J_M^\top J_M$. Then
\begin{equation}\label{eq:whitened-proxy}
\begin{aligned}
\lambda_{\max}(A,B)
&=
\lambda_{\max}\left(B^{-1/2} A B^{-1/2}\right)  \\
&=
\lambda_{\max}\left(
 (J_M B^{-1/2})^\top (J_M B^{-1/2})
\right)                                      \\
&=
\|J_M B^{-1/2}\|_{\mathrm{op}}^2 .
\end{aligned}
\end{equation}
The matrix $B^{-1/2}$ whitens the local paraphrase coordinate against the
proxy metric. Hence the unit budget set $\{u^\top Bu\leq 1\}$ becomes the
Euclidean unit ball $\{z:\|z\|_2\leq 1\}$, and the target Jacobian becomes
$J_M B^{-1/2}$. The attackability index is therefore the squared operator
norm of this whitened Jacobian: the adversary first equalises all paraphrase
directions by the proxy budget and then asks how much the target embedding
stretches the resulting whitened directions.
\end{remark}

\begin{proposition}[Chart invariance]\label{prop:chart-inv}
Let $\phi: \tilde U \to U$ be a $C^2$ diffeomorphism with $\phi(0) = 0$,
and let $T:= D\phi(0) \in GL(q, \mathbb{R})$. Writing
$\tilde E_M:= E_M \circ \phi$ and $\tilde E_P:= E_P \circ \phi$,
and denoting by $\tilde A, \tilde B$ the matrices of
Lemma~\ref{lem:localrep} computed in the chart $\tilde u$, we have
\begin{equation}\label{eq:congruence}
 \tilde A = T^\top A T, \qquad \tilde B = T^\top B T.
\end{equation}
When  $B$ is   invertible,
\begin{equation}\label{eq:similarity}
 \tilde B^{-1} \tilde A = T^{-1} (B^{-1} A) T,
\end{equation}
  so the generalised eigenvalues of $(A,B)$ are independent of the
chosen paraphrase chart. Consequently, the local attackability index
\begin{equation}\label{eq:lambda-star}
 \lambda^*(x):= \lambda_{\max}(A,B)
\end{equation}
is intrinsic to the pair of local pullback metrics induced by
$(E_M,E_P)$ at $x$. If  $u^*$   solves $A u^* = \lambda^* B u^*$, then
$\tilde u^*:= T^{-1}u^*$ solves the corresponding eigenproblem in the
chart $\tilde u$, and the target-space tangent vector
$J_Mu^* = \tilde J_M\tilde u^*$ is unchanged.
\end{proposition}

In other words, changing the local parametrisation of paraphrase directions merely changes the coordinate representation of the two quadratic forms; it does not change the relative spectrum of the proxy and target geometries, the attackability index, or the induced first-order target displacement.

  \begin{remark}[Coordinate-free interpretation]\label{rem:coord-free}
Proposition~\ref{prop:chart-inv} shows that $\lambda^*(x)$ is not an
artefact of the particular coordinates used to index paraphrases.
Geometrically, $A$ and $B$ are pullbacks of the Euclidean metric through
$E_M$ and $E_P$, and the spectrum of the matrix pencil $(A,B)$ records
one local metric relative to the other.
\end{remark}

\begin{remark}[Intrinsic target-space sensitivity]\label{rem:S-invariant}
The same reparameterisation calculation shows that the target-space sensitivity
matrix
\[
S(x):=J_M(x)B(x)^{-1}J_M(x)^\top
\]
is also intrinsic to the local pair of embedders. Indeed, under
$\tilde J_M=J_MT$ and $\tilde B=T^\top BT$,
\[
\tilde J_M\tilde B^{-1}\tilde J_M^\top
=(J_MT)(T^\top BT)^{-1}(T^\top J_M^\top)
=J_MB^{-1}J_M^\top.
\]
Consequently the readout-dependent quantity
$w^\top S(x)w$ and the adjusted margins used below are independent of the
chosen paraphrase chart. Moreover,
\[
S(x)=(J_M(x)B(x)^{-1/2})(J_M(x)B(x)^{-1/2})^\top,
\]
which exhibits $S(x)$ as symmetric and positive semidefinite. Applied to $S(x)$, the matrix-algebra
identity that $CD$ and $DC$ share the same nonzero spectrum, with
$C=J_M(x)B(x)^{-1/2}$ and $D=B(x)^{-1/2}J_M(x)^\top$, gives
\[
\|S(x)\|_{\mathrm{op}}
=\lambda_{\max}(S(x))
=\lambda_{\max}\left(B(x)^{-1/2}A(x)B(x)^{-1/2}\right)
=\lambda^*(x).
\]
\end{remark}

\begin{remark}[Source-side and target-side descriptions]\label{rem:dual-views}
The identity $\|S(x)\|_{\mathrm{op}}=\lambda^*(x)$ has a direct geometric
reading. Two equivalent questions can be asked about the worst-case local
effect of a proxy-bounded paraphrase. The source-side question fixes the
paraphrase budget in the local chart and asks which paraphrase direction
induces the largest squared target displacement; the answer is
$\eta^2\lambda^*(x)$, attained by the leading generalised eigenvector of the
pencil $(A(x),B(x))$. The target-side question fixes a unit direction $w$
in the target embedding space and asks how much target displacement along
$w$ can be achieved by any proxy-bounded paraphrase; the maximum over $w$ of
this displacement is $\eta\sqrt{\|S(x)\|_{\mathrm{op}}}$. The two questions
are dual descriptions of the same worst-case event, one viewed from the
paraphrase side and one from the embedding side, and the identity expresses
the agreement of their answers. The source-side description is convenient
when reasoning about the pencil and its generalised eigenvectors; the
target-side description is convenient when reasoning about the readout and
the adjusted margin.
\end{remark}

\section{Local prediction flips}\label{sec:Prediction-Flip}

Proposition~\ref{prop:solution} bounds the worst-case displacement
of the target embedding under a proxy budget $\eta$. We now translate
this into a condition on the adversary's ability to flip the   affine
readout's prediction.   Let
\begin{equation}\label{eq:s0}
 s_0:= w^\top E_M(0)+b,
 \qquad
 \gamma_w(x):=|s_0|.
\end{equation}
A flip requires movement toward the decision boundary. If $s_0>0$ the
attacker must move the score in the negative direction; if $s_0<0$ the
attacker must move it in the positive direction.

Under the local representation~\eqref{eq:smooth-charts}, the   score
change induced by $u$ is
 \begin{equation}
 w^\top \left( E_M(u) - E_M(0) \right) = w^\top J_M u + o(\|u\|).
 \end{equation}
The exact calculation below is therefore a statement about the
linearised local model
 \begin{equation}  \label{eq:linearized-score}
 s_{\rm lin}(u)=s_0+ w^\top J_Mu
  \qquad\text{under}\qquad u^\top Bu\leq \eta^2.
\end{equation}

  \begin{theorem}[Linearised local prediction-flip condition] \label{thm:flip}
Fix a base text $x$ and assume the conditions of Lemma~\ref{lem:localrep}
with $B(x)\succ 0$. Define
\begin{equation}\label{eq:Sigma-w}
 \Sigma_w(x):= w^\top J_M(x) B(x)^{-1} J_M(x)^\top w.
\end{equation}
In the following identities, suppress the $x$-dependence and write
$J_M=J_M(x)$ and $B=B(x)$. Then
\begin{equation}\label{eq:abs-flip-disp}
 \sup_{u^\top Bu\leq \eta^2}|w^\top J_Mu|
 = \eta\sqrt{\Sigma_w(x)}.
\end{equation}
If $\Sigma_w(x)>0$, one optimal linearised direction toward the decision
boundary is
 \begin{equation}  \label{eq:uwstar}
 u^*_{\rm flip}
 =-\operatorname{sign} (  s_0 )  \,\eta\,
 \frac{B^{-1}J_M^\top w}{\sqrt{\Sigma_w(x)}}.
 \end{equation}
  If $\Sigma_w(x)=0$, the supremum in \eqref{eq:abs-flip-disp} is zero.
Consequently, when $\gamma_w(x)>0$, a  prediction flip at $x$ is achievable
within proxy budget $\eta$ in the linearised local model
\eqref{eq:linearized-score} if and only if
\begin{equation}\label{eq:flip-condition}
 \eta \sqrt{\Sigma_w(x)} > \gamma _w (x).
\end{equation}
If $\gamma_w(x)=0$, the base input already lies on the decision hyperplane.
The strict flip formulation is then degenerate and depends on the convention for
ties at the boundary, although \eqref{eq:abs-flip-disp} still gives the maximum
first-order score displacement. The boundary set $\{x:s_{w,b}(x)=0\}$ has
probability zero under any distribution absolutely continuous on the
target embedding, and the strict-versus-non-strict form of
\eqref{eq:flip-condition} is irrelevant to the population statements
that follow.
\end{theorem}

\begin{theorem}[Finite-radius nonlinear error]\label{thm:linearization-error}
Theorem~\ref{thm:flip} is exact for the linearised model. For the
original smooth maps, define
\[
 r_M(\eta):=\sup_{u^\top Bu\leq \eta^2}
 \left|w^\top\{E_M(u)-E_M(0)-J_Mu\}\right|.
\]
Then $\eta\sqrt{\Sigma_w(x)}>\gamma_w(x)+r_M(\eta)$ is sufficient for a
flip in the relaxed local model, while
$\eta\sqrt{\Sigma_w(x)}<\gamma_w(x)-r_M(\eta)$ rules out a flip over the
same ellipsoid. If the second derivative of $E_M$ is bounded and
$B\succeq \beta I$, then $r_M(\eta)=O(\eta^2/\beta)$ as $\eta\downarrow0$.
The factor $1/\beta$ enters because the proxy ellipsoid
$\{u:u^\top Bu\leq\eta^2\}$ is contained in the Euclidean ball of radius
$\eta/\sqrt{\beta}$, so that the quadratic Taylor remainder of $E_M$
is controlled by $\|u\|_2^2\leq\eta^2/\beta$ on the feasible set.
An analogous second-order error is incurred when replacing the exact
proxy constraint by $u^\top Bu\leq\eta^2$.
\end{theorem}

 The quantity $\Sigma_w(x)$ measures the squared maximum displacement of
the target embedding   along the readout  direction $w$ per unit proxy
budget. Theorem~\ref{thm:sigma-lambda} shows that this   classifier-specific
quantity  is bounded above by the local attackability index
  $\lambda^*(x)$  of Proposition~\ref{prop:solution}.

\begin{theorem}[Attackability index dominates classifier-direction displacement]\label{thm:sigma-lambda}
Under the assumptions of Lemma~\ref{lem:localrep} with   $B \succ 0$, the
quantity $\Sigma_w(x)$ defined in \eqref{eq:Sigma-w} satisfies, for every
unit vector $w \in \mathbb{R}^d_M$,
\begin{equation}\label{eq:sigma-leq-lambda}
 \Sigma_w(x) \leq \lambda^*(x),
\end{equation}
where
  \[
 \lambda^*(x)=\lambda_{\max}(A,B)
 =\lambda_{\max}\left(B^{-1/2}AB^{-1/2}\right)
 =\|S(x)\|_{\mathrm{op}}.
\]
 Equality in \eqref{eq:sigma-leq-lambda} holds when $w$ is a top
eigenvector of   $J_MB^{-1}J_M^\top$.
\end{theorem}

  \begin{corollary}[Worst-direction no-flip certificate] \label{cor:worstcase-flip}
Under the conditions of Theorem~\ref{thm:flip} and
Theorem~\ref{thm:sigma-lambda},  for a fixed unit readout direction $w$,
 a sufficient condition for a local prediction flip at $x$ to be
unachievable within proxy budget $\eta$   in the linearised local model is
\[
 \eta \sqrt{\lambda^*(x)} \leq \gamma_w(x).
\]
 This certificate is conservative: failure of the inequality does not
imply that the particular readout direction $w$ is flippable, since
$\Sigma_w(x)$ may be much smaller than $\lambda^*(x)$.
 \end{corollary}

\section{Population Attackability}\label{sec:population}

So far the analysis has fixed a single input $x$  and a single affine
readout $(w,b)$. We now lift the   linearised  prediction-flip condition of
Theorem~\ref{thm:flip} to a population-level   statement for this fixed
readout.  Let $\mathcal{D}$ be a distribution on $\mathcal{X}$, write
$\mathbb{P}$ for probability under $\mathcal{D}$, and let $X$ denote a
random element with distribution $\mathcal{D}$.  Throughout
Sections~\ref{sec:population} to~\ref{sec:margins}, we ignore the
degenerate boundary case in which $\Sigma_w(x)=0$ and $\gamma_w(x)=0$
simultaneously, or assume it has probability zero under $\mathcal D$ for
the readouts under consideration; the convention $Z_w(x)=0$ at this
configuration is then irrelevant for population and uniform results.

\begin{definition}[Attackability margin]\label{def:Z}
The \emph{attackability-adjusted margin} of $X$ for the fixed readout
$(w,b)$ is the extended nonnegative random variable
\begin{equation}\label{eq:Z}
 Z_w(x):=
   \begin{cases}
 \gamma_w(x)/\sqrt{\Sigma_w(x)}, & \Sigma_w(x)>0,\\[3pt]
 +\infty, & \Sigma_w(x)=0 \text{ and } \gamma_w(x)>0,\\[3pt]
 0, & \Sigma_w(x)=0 \text{ and } \gamma_w(x)=0.
 \end{cases}
\end{equation}
The \emph{population attackability} at proxy budget $\eta>0$ is
\begin{equation}\label{eq:Aeta}
 \mathcal{A}_w(\eta):=
 \mathbb{P}\left[Z_w(X)<\eta\right]
 =
 F_{Z_w}^{-}(\eta).
\end{equation}
where $F_{Z_w}^{-}(\eta):=\mathbb{P}[Z_w(X)<\eta]=F_{Z_w}(\eta-)$ is
the left-limit version of the distribution function. We use the strict
inequality because the strict flip condition in Theorem~\ref{thm:flip} gives
$\eta\sqrt{\Sigma_w(x)}>\gamma_w(x)$. If $Z_w$ has no atom at $\eta$, this
agrees with the usual CDF $F_{Z_w}(\eta)=\mathbb{P}[Z_w(X)\leq\eta]$.
\end{definition}

To understand the intuition behind Definition~\ref{def:Z}, recall from
Theorem~\ref{thm:flip} that, in the linearised local model,  a paraphrase
can flip the   readout  on input $x$ if and only if
  $\eta\sqrt{\Sigma_w(x)} > \gamma_w(x)$. Rearranging, the flip is
achievable if and only if   $\eta > \gamma_w(x)/\sqrt{\Sigma_w(x)}$ when
$\Sigma_w(x)>0$, and the right-hand side is exactly   $Z_w(x)$. Thus
$Z_w(x)$  is the smallest proxy budget at which the adversary can flip the
 linearised  prediction on $x$: a large   $Z_w(x)$ implies local robustness,
and a small $Z_w(x)$ implies local fragility. Inputs with
$\Sigma_w(x)=0$ and $\gamma_w(x)>0$ are first-order certifiably robust in
the readout direction, since no proxy-bounded perturbation changes the
score to first order.

In the population setting, $X$ is drawn from $\mathcal{D}$ and $Z_w(X)$ becomes a random variable. The quantity $\mathcal{A}_w(\eta)=\mathbb{P}_{x\sim\mathcal{D}}[Z_w(X)<\eta]$ is the
population  fraction of inputs that are flippable at budget $\eta$ in the
linearised local model. It is non-decreasing in $\eta$, satisfies
$\mathcal{A}_w(0)=0$, and obeys $\lim_{\eta\to\infty}\mathcal{A}_w(\eta)=\mathbb{P}[Z_w(X)<\infty]$. The remainder of this section bounds $\mathcal{A}_w(\eta)$ in terms of margin and anisotropy tails and estimates it from finite samples.

\begin{theorem}[Margin-tail bound on population attackability]\label{thm:Aeta-margin}
Assume the conditions of Theorem~\ref{thm:flip} hold $\mathbb{P}$-a.s.
  Then, for every $\eta>0$ and every deterministic $\Lambda>0$,
\begin{equation}\label{eq:Aeta-bound-tail}
 \mathcal{A}_w(\eta)
 \leq
 \mathbb{P}\left[\gamma_w(X)<\eta\sqrt{\Lambda}\right]
 +\mathbb{P}\left[\lambda^*(X)>\Lambda\right].
\end{equation}
In particular, if $\lambda^*(X)\leq\Lambda$ $\mathbb{P}$-a.s., then
 \begin{equation}\label{eq:Aeta-bound}
 \mathcal{A} _w (\eta)
  \leq
  \mathbb{P}\left[\gamma _w (X)<\eta\sqrt{\Lambda}\right]
  =  F  _{\gamma_w}^{-} \left(\eta\sqrt{\Lambda}\right),
\end{equation}
where $F_{\gamma_w}^{-}(t):=\mathbb{P}[\gamma_w(X)<t]=F_{\gamma_w}(t-)$. Equivalently, if
$\Lambda_{1-\beta}$ is chosen so that
$\mathbb{P}[\lambda^*(X)>\Lambda_{1-\beta}]\leq\beta$, then
\begin{equation}\label{eq:Aeta-quantile}
 \mathcal{A}_w(\eta)
 \leq
 \mathbb{P}\left[\gamma_w(X)<\eta\sqrt{\Lambda_{1-\beta}}\right]
 +\beta.
\end{equation}
 \end{theorem}

\begin{remark}[Readout-specific quantile bound]\label{rem:Aeta-sigma}
The bound of Theorem~\ref{thm:Aeta-margin} controls $\mathcal{A}_w(\eta)$
through the worst-direction index $\lambda^*(x)$. Because
$\mathcal{A}_w(\eta)=\mathbb{P}[\gamma_w(X)<\eta\sqrt{\Sigma_w(X)}]$ by
Definition~\ref{def:Z}, the identical argument applied to $\Sigma_w(X)$
in place of $\lambda^*(X)$ gives, for every deterministic $\Lambda>0$,
\[
\mathcal{A}_w(\eta)\le
\mathbb{P}\!\left[\gamma_w(X)<\eta\sqrt{\Lambda}\right]
+\mathbb{P}\!\left[\Sigma_w(X)>\Lambda\right],
\]
and, if $\Sigma_{1-\beta}$ is chosen so that
$\mathbb{P}[\Sigma_w(X)>\Sigma_{1-\beta}]\le\beta$,
\[
\mathcal{A}_w(\eta)\le
\mathbb{P}\!\left[\gamma_w(X)<\eta\sqrt{\Sigma_{1-\beta}}\right]+\beta.
\]
Since $\Sigma_w(x)\le\lambda^*(x)$ by Theorem~\ref{thm:sigma-lambda}, the
$(1-\beta)$-quantile of $\Sigma_w$ is no larger than that of $\lambda^*$,
so this readout-specific bound is never looser than the readout-free
bound of Theorem~\ref{thm:Aeta-margin} at the same $\beta$.
\end{remark}

\begin{theorem}[Empirical concentration of population attackability]\label{thm:Aeta-empirical}
  Fix the affine readout $(w,b)$. Let $X_1,\ldots,X_n$ be i.i.d. draws from
$\mathcal{D}$, and define the empirical attackability  curve
 \begin{equation}\label{eq:Aeta-empirical}
 \hat{\mathcal{A}}  _{n,w} (\eta):= \frac{1}{n}\sum_{i=1}^n \mathbbm{1}\left[Z _w (X_i) < \eta\right].
\end{equation}
Then, for every   $\delta \in (0,1)$, with probability at least
  $1-\delta$  over the sample,
\begin{equation}\label{eq:Aeta-dkw}
 \sup  _{\eta>0} \left|\hat{\mathcal{A}}  _{n,w} (\eta)-\mathcal{A} _w (\eta)\right|
 \leq \sqrt{\frac{\ln(2/\delta)}{2n}}.
\end{equation}
\end{theorem}

 \begin{remark}[Fixed versus data-dependent readouts]\label{rem:fixed-readout}
Theorem~\ref{thm:Aeta-empirical} is an oracle concentration statement for
a fixed readout and exactly observed values of $Z_w(X_i)$. If the readout
is learned or selected using the same sample, the DKW bound does not by
itself justify uniform validity over the class of possible readouts.
That extension requires sample splitting or a uniform-convergence bound
for the class of attackability events
$\{x\mapsto\mathbbm{1}[Z_w(x)<\eta]:(w,b,\eta)\in\mathcal{W}\times\mathbb{R}\times(0,\infty)\}$.
\end{remark}

\section{Uniform Attackability over Classes of Affine Readouts}\label{sec:hypotheses}

Section~\ref{sec:population} treats the affine readout $(w,b)$ as fixed. This is the right setting for auditing a single deployed classifier, and it is also the setting of the empirical exercise in Section~\ref{sec:experiments}, where FinBERT is treated as a deployed financial-sentiment classifier with a fixed pre-trained affine readout. However, in statistical learning the readout is typically trained, tuned, or otherwise
selected from data. A pointwise concentration statement such as
Theorem~\ref{thm:Aeta-empirical} does not by itself justify using the same
sample both to choose the readout and to estimate its attackability. The goal
of this section is to replace the fixed-readout DKW bound by a uniform bound
over a class of affine readouts.

Throughout this section we write
\[
z(x):=e_M(x), \qquad
S(x):=J_M(x)B(x)^{-1}J_M(x)^\top,
\]
where $J_M(x)$ and $B(x)$ are the local target Jacobian and proxy pullback
metric at $x$. By Remark~\ref{rem:S-invariant}, $S(x)$ does not depend on the
particular local chart used to parameterise paraphrases. The matrix $S(x)$ is
positive semidefinite and records how proxy-valid paraphrase directions move the
target representation in the coordinates seen by the readout. For a readout
parameter
$\theta=(w,b)$, define
\[
f_\theta(x):=w^\top z(x)+b,\qquad
\gamma_\theta(x):=|f_\theta(x)|,\qquad
\Sigma_\theta(x):=w^\top S(x)w.
\]
The linearised local attackability event at budget $\eta$ is
\begin{equation}\label{eq:attack-event-theta}
G_{\theta,\eta}(x)
:=\mathbbm{1}\left\{|f_\theta(x)|<\eta\sqrt{\Sigma_\theta(x)}\right\}.
\end{equation}
Thus $G_{\theta,\eta}(x)=1$ means that, in the first-order local model,
the margin of the readout is smaller than the displacement that the adversary
can induce in the readout direction. For a class $\Theta$ of affine readouts,
set
\begin{equation}\label{eq:G-class}
\mathcal G_\Theta
:=\left\{x\mapsto G_{\theta,\eta}(x):\theta\in\Theta,\,\eta>0\right\}.
\end{equation}
The corresponding population and empirical attackability curves are
\[
\mathcal A_\theta(\eta):=\mathbb P\{G_{\theta,\eta}(X)=1\},
\qquad
\widehat{\mathcal A}_{n,\theta}(\eta)
:=\frac1n\sum_{i=1}^nG_{\theta,\eta}(X_i).
\]
The empirical and population attackability curves
$\widehat{\mathcal A}_{n,\theta}(\eta)$ and $\mathcal A_\theta(\eta)$ are the
same objects as $\widehat{\mathcal A}_{n,w}(\eta)$ and $\mathcal A_w(\eta)$ in
Section~\ref{sec:population}, with the subscript $\theta=(w,b)$ replacing $w$
to make the dependence on the bias explicit. Definition~\ref{def:Z} and
Theorem~\ref{thm:flip} give
$G_{\theta,\eta}(x)=\mathbbm{1}\{Z_\theta(x)<\eta\}$ under the conventions
already in place for $\Sigma_\theta(x)=0$, so the two notations refer to the
same indicator function and to the same empirical and population averages.
Throughout this section and Section~\ref{sec:margins} we assume that
$\mathbb P\{\gamma_\theta(X)=0\}=0$ for each readout
$\theta\in\Theta$; this excludes the boundary-degeneracy case where
both $\gamma_\theta(x)$ and $\Sigma_\theta(x)$ vanish, on which the
strict inequality in \eqref{eq:attack-event-theta} and the convention
$Z_\theta(x)=0$ in Definition~\ref{def:Z} would otherwise disagree.

The next theorem is a direct uniform-convergence analogue of the DKW result
in Section~\ref{sec:population}. Instead of controlling the empirical CDF of
one fixed random variable $Z_w(X)$, it controls all attackability indicators
in $\mathcal G_\Theta$ simultaneously. We use the VC dimension \citep{vapnik1971uniform} as the complexity measure of binary function classes.
For a class $\mathcal F$ of $\{0,1\}$-valued functions on a domain $\mathcal Z$, a
finite set $\{z_1,\ldots,z_k\}\subset\mathcal Z$ is \emph{shattered} by $\mathcal F$
if, for every $\sigma\in\{0,1\}^k$, some $f\in\mathcal F$ satisfies
$f(z_i)=\sigma_i$ for all $i$. The VC dimension $\VC(\mathcal F)$ is the largest
cardinality of a shattered set, with $\VC(\mathcal F)=\infty$ if shattered sets
of unbounded size exist. It controls the rate of uniform convergence of
empirical to population frequencies and appears in distribution-free
generalisation bounds.

\begin{theorem}[Uniform VC concentration for attackability]\label{thm:vc-uniform}
Let $X_1,\ldots,X_n$ be i.i.d. from $\mathcal D$, and suppose that
$\VC(\mathcal G_\Theta)=v<\infty$. There is a universal constant $C$ such
that, for every $\delta\in(0,1)$, with probability at least $1-\delta$,
\begin{equation}\label{eq:vc-uniform-bound}
\sup_{\theta\in\Theta,\,\eta>0}
\left|
\widehat{\mathcal A}_{n,\theta}(\eta)-\mathcal A_\theta(\eta)
\right|
\leq
C\sqrt{\frac{v\log(en)+\log(1/\delta)}{n}}.
\end{equation}
\end{theorem}

Theorem~\ref{thm:vc-uniform} says that empirical attackability remains a
valid diagnostic even when the affine readout is chosen after seeing the
sample, provided the induced class of attackability events has finite VC
dimension. This is the learning-theoretic upgrade from a fixed classifier to
a data-dependent classifier. The DKW bound is sharper for a single fixed
readout, but it has no built-in protection against data-dependent choice of
$\theta$.

%\begin{remark}[The budget parameter alone is not the source of complexity]
%For a fixed readout $\theta$, the subclass
%$\{G_{\theta,\eta}:\eta>0\}$ is only a one-parameter threshold family. Writing
%$Z_\theta(x)=|f_\theta(x)|/\sqrt{\Sigma_\theta(x)}$ with the same conventions as
%Definition~\ref{def:Z}, we have
%$G_{\theta,\eta}(x)=\mathbbm{1}\{Z_\theta(x)<\eta\}$.
%Thus this subclass has VC dimension at most $2$ as a class of one-dimensional
%thresholds. Applying Theorem~\ref{thm:vc-uniform} with $\theta$ fixed gives a
%rate of order $\sqrt{\log n/n}$, matching the DKW fixed-readout rate of
%Theorem~\ref{thm:Aeta-empirical} up to the logarithmic factor introduced by the
%general VC inequality. The dimension-dependent bounds below
%(Propositions~\ref{prop:vc-affine} and~\ref{prop:vc-improved}) are therefore
%driven by uniformity over the readout, not by the scalar budget parameter.
%\end{remark}

The following proposition gives a simple dimension-dependent VC bound for
all affine readouts on a $d$-dimensional target embedding. The proof is based
on one observation: the attackability event is a quadratic inequality in the
embedding coordinates, together with the entries of the local sensitivity
matrix $S(x)$, and such inequalities can be viewed as linear thresholds after
an elementary feature lift. This is the degree-two instance of the standard
polynomial-surface viewpoint: a polynomial threshold can be represented as an
affine threshold in a higher-dimensional space of monomial features
\citep{anthony1995classification}.

\begin{proposition}[A VC bound for affine readouts]\label{prop:vc-affine}
Assume $z(x)\in\mathbb R^{d_M}$ and $S(x)$ is a symmetric positive semidefinite
$d_M\times d_M$ matrix for every $x$. Let $\Theta\subseteq\mathbb R^{d_M}\times\mathbb R$
be any set of affine-readout parameters $\theta=(w,b)$. Then
\begin{equation}\label{eq:vc-aff-bound}
\VC(\mathcal G_\Theta)\leq (d_M+1)^2.
\end{equation}
\end{proposition}

The bound in Proposition~\ref{prop:vc-affine} is conservative, since the
elementary lifting treats the structured quadratic coefficients as if they
were independent. A polynomial-threshold argument that exploits the
parameterisation of $\mathcal G_\Theta$ by $(w,b,\eta)$ gives a sharper rate.

\begin{proposition}[Polynomial-threshold improvement]\label{prop:vc-improved}
Under the assumptions of Proposition~\ref{prop:vc-affine},
\[
\VC(\mathcal G_\Theta)=O(d_M).
\]
\end{proposition}

\begin{corollary}[Uniform attackability bound for affine readouts]\label{cor:affine-vc-uniform}
Under the assumptions of Proposition~\ref{prop:vc-affine}, there is a
universal constant $C$ such that, with probability at least $1-\delta$,
\begin{equation}\label{eq:affine-vc-uniform}
\sup_{\theta\in\Theta,\,\eta>0}
\left|
\widehat{\mathcal A}_{n,\theta}(\eta)-\mathcal A_\theta(\eta)
\right|
\leq
C\sqrt{\frac{(d_M+2)\log(en)+\log(1/\delta)}{n}}.
\end{equation}
The bound follows from Theorem~\ref{thm:vc-uniform} combined with the
polynomial-threshold VC bound of Proposition~\ref{prop:vc-improved}.
\end{corollary}

\section{Attackability-Adjusted Margin Bounds}\label{sec:margins}

The VC bound above controls the binary attackability indicator directly.
Classical margin theory suggests a complementary approach: rather than count
only whether an input is attackable, measure how attackable it is. This
replaces a binary event by a real-valued margin distribution, which can give
sharper and more informative generalization bounds.

The attackability event in Section~\ref{sec:hypotheses} is label-free: it
asks only whether the classifier's prediction can be flipped, not whether
the prediction is correct. The margin theory below introduces labels.
Assume that the data are labelled, with $Y\in\{-1,+1\}$. For a readout
$\theta=(w,b)$ and budget $\eta>0$, we define the \emph{attackability-adjusted
margin} as follows. Recall from Section~\ref{sec:hypotheses} the notation $z(x):=e_M(x)$ for the
target embedding and $S(x):=J_M(x)B(x)^{-1}J_M(x)^\top$ for the local
sensitivity matrix.  
\begin{equation}\label{eq:adjusted-margin}
m_{\eta,\theta}(x,y)
:=y\bigl(w^\top z(x)+b\bigr)-\eta\sqrt{w^\top S(x)w}.
\end{equation}
The first term is the usual signed classifier margin. The second term is
the largest first-order displacement, at proxy budget $\eta$, that any
proxy-bounded paraphrase can induce against this readout direction. The
adjusted margin therefore subtracts a local geometric penalty from the
ordinary margin, reflecting the worst-case loss of margin available within
the proxy budget.

The linearised local robust risk is
\begin{equation}\label{eq:robust-risk}
R_\eta(\theta):=\mathbb P\{m_{\eta,\theta}(X,Y)\leq 0\},
\end{equation}
where a positive adjusted margin certifies correct classification after the
linearised proxy-bounded perturbation.

The flip event and the robust-risk event behave differently and it is
worth recording how. The flip event from
Section~\ref{sec:hypotheses},
\[
\{Z_\theta(X)<\eta\}=\{\gamma_\theta(X)<\eta\sqrt{\Sigma_\theta(X)}\},
\]
asks whether some proxy-bounded paraphrase moves $X$ across the decision
boundary of the readout. It depends on $X$ but not on the label $Y$.
The robust-risk event,
\[
\{m_{\eta,\theta}(X,Y)\leq 0\}=\{Yf_\theta(X)\leq\eta\sqrt{\Sigma_\theta(X)}\},
\]
asks whether the readout either already misclassifies $X$ or can be made
to misclassify it by such a paraphrase. It depends on both $X$ and $Y$.

On a correctly classified input ($Yf_\theta(X)>0$),
$Yf_\theta(X)=|f_\theta(X)|=\gamma_\theta(X)$, and the two events become
$\{\gamma_\theta(X)<\eta\sqrt{\Sigma_\theta(X)}\}$ and
$\{\gamma_\theta(X)\leq\eta\sqrt{\Sigma_\theta(X)}\}$ respectively, so they
agree apart from the equality boundary.

On a misclassified input ($Yf_\theta(X)<0$), the two events diverge. The
robust-risk event automatically holds, because $Yf_\theta(X)<0$ and
$\eta\sqrt{\Sigma_\theta(X)}\geq 0$, so the inequality
$Yf_\theta(X)\leq\eta\sqrt{\Sigma_\theta(X)}$ is satisfied regardless of
the adversary. For the flip event, suppose $Y=+1$ but $f_\theta(X)<0$, so
that the classifier predicts $-1$ on a positive-labelled input. A
paraphrase that moves the readout score $f_\theta$ from negative through
zero into positive territory would "flip" the prediction, and the flip
event would treat this as an adversarial success. But the flipped
prediction is now $+1$, which is the correct label. The flip event has
fired on what is, from the user's point of view, a correction rather than
an attack. The robust-risk event ignores such cases: it already counts
the input as an error before the adversary acts, and a paraphrase that
corrects the error does not remove that count.

The adjusted margin therefore tracks robust classification risk, not
prediction-flip frequency. The two coincide on correctly classified
inputs and disagree on misclassified inputs, where robust risk is the
right quantity for a learning-theoretic bound.

For a \emph{margin slack}  $\rho>0$, its empirical $\rho$-margin analogue is
\begin{equation}\label{eq:emp-margin-risk}
\widehat R_{n,\eta,\rho}(\theta)
:=\frac1n\sum_{i=1}^n
\mathbbm{1}\left\{m_{\eta,\theta}(X_i,Y_i)\leq \rho\right\}.
\end{equation}
The slack parameter $\rho$ plays the same role as in ordinary margin bounds.
The empirical $\rho$-margin error counts two kinds of examples: those with
adjusted margin at most zero, which are already locally non-robust on the
sample, and those with adjusted margin in the interval $(0,\rho]$, which are
empirically robust but only by an amount small enough that uniform
fluctuation between the empirical and population margin distributions could
plausibly move them across the decision threshold. 
 The slack $\rho$ acts as a safety buffer. A wider slack counts more sample
points as potentially fragile, but reduces the complexity penalty in
Theorem~\ref{thm:fat-margin}, which depends on the fat-shattering dimension
of the adjusted-margin class at scale $\rho$ and is decreasing in $\rho$.

Let
\begin{equation}\label{eq:M-class}
\mathcal M_\eta
:=\left\{(x,y)\mapsto m_{\eta,\theta}(x,y):\theta\in\Theta\right\}.
\end{equation}
The following theorem is the standard fat-shattering margin bound applied to
this adjusted-margin class. One route to this form
is to combine the margin-covering bound of \citet[Theorem~10.4]{anthony2009neural}
with the fat-shattering covering estimate of
\citet[Theorem~12.13]{anthony2009neural}.

We use the fat-shattering dimension as the scale-sensitive complexity
measure of real-valued function classes. For a class $\mathcal F$ on a
domain $\mathcal Z$ and $\alpha>0$, a finite set $\{z_1,\ldots,z_k\}\subset
\mathcal Z$ is \emph{$\alpha$-shattered} by $\mathcal F$ if there exist
reference levels $r_1,\ldots,r_k\in\mathbb R$ such that, for every
$\sigma\in\{-1,+1\}^k$, some $f\in\mathcal F$ satisfies
$\sigma_i(f(z_i)-r_i)\geq\alpha$ for all $i$. The fat-shattering dimension
$\fat_\alpha(\mathcal F)$ is the largest cardinality of an $\alpha$-shattered
set. It is a scale-sensitive analogue of VC dimension and appears in margin
bounds because it controls the covering number of $\mathcal F$ at resolution
$\alpha$.

\begin{theorem}[Fat-shattering margin bound]\label{thm:fat-margin}
Assume that every function in $\mathcal M_\eta$ takes values in $[-M,M]$.
Fix $\rho\in(0,M]$, and let
\[
d_\rho:=\fat_{\rho/8}(\mathcal M_\eta).
\]
There is a universal constant $C$ such that, for i.i.d. samples
$(X_i,Y_i)_{i=1}^n$, with probability at least $1-\delta$, uniformly over
$\theta\in\Theta$,
\begin{equation}\label{eq:fat-margin-bound}
R_\eta(\theta)
\leq
\widehat R_{n,\eta,\rho}(\theta)
+
C\sqrt{\frac{d_\rho\log^2(CMn/\rho)+\log(1/\delta)}{n}}.
\end{equation}
\end{theorem}

This theorem should be interpreted exactly like a classical margin bound.
If most sample points have adjusted margin much larger than $\rho$, then the
empirical term is small. The complexity term measures the cost of uniform
control over the readout class $\Theta$. The relevant margin is not
$y f_\theta(x)$ alone, but $y f_\theta(x)-\eta\sqrt{w^\top S(x)w}$: points
with large ordinary margin can still be fragile if their local displacement
term is also large.

For affine readouts, one can also obtain a more concrete norm-controlled
covering bound. We use the following covering-number notation. If
\(\mathcal F\) is a class of real-valued functions on a domain \(\mathcal Z\),
the \emph{supremum-norm distance} between two functions $f,g\in\mathcal F$ is
\[
d_\infty(f,g):=\sup_{z\in\mathcal Z}|f(z)-g(z)|.
\]
For \(\epsilon>0\), \(\mathcal N_\infty(\epsilon,\mathcal F)\) denotes the
supremum-norm covering number, i.e., the least cardinality of a finite set
\(\mathcal C\subseteq\mathcal F\) such that for every \(f\in\mathcal F\)
there is a \(g\in\mathcal C\) with \(d_\infty(f,g)\leq\epsilon\). If no
finite such cover exists, we set \(\mathcal N_\infty(\epsilon,\mathcal F)=\infty\).
In the proposition below, the domain is \(\mathcal Z=\mathcal X\times\{-1,+1\}\),
since functions in \(\mathcal M_\eta\) take labelled examples \((x,y)\) as
input.

Let
\[
\Theta_{W,B_0}:=\{(w,b):\|w\|_2\leq W,
\ |b|\leq B_0\}.
\]
The next proposition shows that the adjusted-margin class is Lipschitz in
$(w,b)$ when the embeddings are bounded and the local anisotropy matrices are
uniformly bounded. The factor $R+\eta\sqrt\Lambda$ is the effective radius
seen by the readout after accounting for adversarial displacement.

\begin{proposition}[Norm-controlled covering bound]\label{prop:cover-margin}
Assume $\|z(x)\|_2\leq R$ and $\|S(x)\|_{\mathrm{op}}\leq \Lambda$ for all
$x$. By Remark~\ref{rem:S-invariant}, this is equivalent to
$\lambda^*(x)\leq\Lambda$, the bounded-anisotropy condition used in
Theorem~\ref{thm:Aeta-margin}. For the class $\mathcal M_\eta$ induced by
$\Theta_{W,B_0}$, the supremum-norm covering number satisfies, for every
$\epsilon>0$,
\begin{equation}\label{eq:cover-margin}
\mathcal N_\infty(\epsilon,\mathcal M_\eta)
\leq
\left(1+\frac{4W(R+\eta\sqrt\Lambda)}{\epsilon}\right)^{d_M}
\left(1+\frac{4B_0}{\epsilon}\right).
\end{equation}
Moreover, every $m\in\mathcal M_\eta$ is bounded in absolute value by
\begin{equation}\label{eq:M-bound}
M_\eta:=WR+B_0+\eta W\sqrt\Lambda.
\end{equation}
\end{proposition}

\begin{corollary}[Norm-controlled attackability-adjusted margin bound]\label{cor:norm-margin}
Under the assumptions of Proposition~\ref{prop:cover-margin}, there is a
universal constant $C$ such that, for every fixed $\rho\in(0,M_\eta]$, with
probability at least $1-\delta$, uniformly over $\theta\in\Theta_{W,B_0}$,
\begin{align}\label{eq:norm-margin-bound}
R_\eta(\theta)
&\leq
\widehat R_{n,\eta,\rho}(\theta)
+ C\sqrt{\frac{
 d_M\log\left(1+\frac{C W(R+\eta\sqrt\Lambda)}{\rho}\right)
 +\log\left(1+\frac{C B_0}{\rho}\right)
 +\log(1/\delta)}{n}}.
\end{align}
\end{corollary}

Corollary~\ref{cor:norm-margin} exhibits a two-sided dependence on the local
geometry. Increasing the budget $\eta$ or the anisotropy bound $\Lambda$
makes the displacement term $\eta\sqrt{w^\top S(x)w}$ larger pointwise, which
moves more sample points into the $\rho$-margin set and so increases the
empirical term $\widehat R_{n,\eta,\rho}(\theta)$. The same quantity also
appears in the effective radius $R+\eta\sqrt\Lambda$ inside the complexity
term, where it inflates the covering number of the adjusted-margin class
and so increases the uniform-convergence contribution. The bound is small
only when both effects are small, i.e., when the budget is modest, the local
anisotropy is well-controlled, and most sample points have adjusted margin
comfortably above $\rho$.

The preceding corollary uses a volumetric covering of the Euclidean ball in
\(\mathbb R^d\). When \(d\) is large, a data-dependent Rademacher bound gives
a more informative alternative. Let \(\sigma_1,\ldots,\sigma_n\) be
i.i.d.\ Rademacher random variables, independent of the sample, taking
values \(\pm1\) with equal probability. For a sample
\(\mathcal S_n=((X_i,Y_i))_{i=1}^n\) and a class \(\mathcal F\) of
real-valued functions on labelled examples, the empirical Rademacher
complexity of \(\mathcal F\) is
\[
\widehat{\mathfrak R}_n(\mathcal F)
:=\mathbb E_\sigma\left[
\sup_{f\in\mathcal F}\frac1n\sum_{i=1}^n\sigma_i f(X_i,Y_i)
\,\middle|\,\mathcal S_n\right],
\]
where the conditional expectation $\mathbb E_\sigma[\cdot\mid\mathcal S_n]$
integrates over the Rademacher variables with the sample held fixed.
Also define the empirical trace sensitivity
\[
\overline T_n:=\frac1n\sum_{i=1}^n\operatorname{tr}S(X_i).
\]
The trace is the right dimension-free quantity for the sensitivity term in
the Rademacher bound. The operator norm $\|S(x)\|_{\mathrm{op}}$ controls
the largest single eigenvalue, i.e., the worst-case sensitivity in any
direction. The trace $\operatorname{tr}S(x)=\sum_k\lambda_k(x)$ sums the
contributions from all eigendirections of $S(x)$, capturing the aggregate
sensitivity across all readout directions that the adversary can move in
locally. The covering bound (Corollary~\ref{cor:norm-margin}) uses the
operator norm because its Lipschitz argument depends on the worst direction;
the Rademacher bound below uses the trace because the empirical-process
sum aggregates across all directions.

\begin{theorem}[Trace-sensitive Rademacher margin bound]
\label{thm:rademacher-margin}
Assume \(S(x)\succeq0\) for every \(x\), and let
\(\mathcal M_\eta\) be the adjusted-margin class induced by
\(\Theta_{W,B_0}\). For every sample \(\mathcal S_n\),
\begin{equation}\label{eq:rad-complexity-bound}
\widehat{\mathfrak R}_n(\mathcal M_\eta)
\leq
C\left\{
\frac{\sqrt{W^2+B_0^2}}{n}
\left(\sum_{i=1}^n(\|z(X_i)\|_2^2+1)\right)^{1/2}
+
\eta W\sqrt{\frac{\overline T_n}{n}}
\right\},
\end{equation}
where \(C\) is a universal constant. Consequently, for every
\(\rho>0\), with probability at least \(1-\delta\), uniformly over
\(\theta\in\Theta_{W,B_0}\),
\begin{equation}\label{eq:rademacher-margin-bound}
\begin{aligned}
R_\eta(\theta)
&\leq
\widehat R_{n,\eta,\rho}(\theta) \\
&\quad +
\frac{C}{\rho}\left\{
\frac{\sqrt{W^2+B_0^2}}{n}
\left(\sum_{i=1}^n(\|z(X_i)\|_2^2+1)\right)^{1/2}
+
\eta W\sqrt{\frac{\overline T_n}{n}}
\right\}
+C\sqrt{\frac{\log(1/\delta)}{n}}.
\end{aligned}
\end{equation}
In particular, if \(\|z(x)\|_2\leq R\) and \(\operatorname{tr}S(x)\leq T\) for
all \(x\), then
\begin{equation}\label{eq:rademacher-margin-bound-deterministic}
\begin{aligned}
R_\eta(\theta)
&\leq
\widehat R_{n,\eta,\rho}(\theta)
+\frac{C}{\rho}\left(
\sqrt{W^2+B_0^2}\sqrt{\frac{R^2+1}{n}}
+\eta W\sqrt{\frac{T}{n}}
\right)
+C\sqrt{\frac{\log(1/\delta)}{n}}.
\end{aligned}
\end{equation}
with the same probability, uniformly over \(\theta\in\Theta_{W,B_0}\).
If, in addition, \(\operatorname{rank}S(x)\leq q\) and
\(\|S(x)\|_{\mathrm{op}}\leq\Lambda\), then one may take \(T=q\Lambda\).
\end{theorem}

This theorem is dimension-free in the target embedding dimension \(d_M\), but it is not
free of geometric complexity. The sensitivity contribution is controlled by
\(\operatorname{tr}S(x)\), or by the effective local rank bound \(q\Lambda\).
This distinction is important: an operator-norm bound alone controls only the
single most expansive proxy direction and does not control the richness of the
whole class \(x\mapsto\sqrt{w^\top S(x)w}\) when many independent local
sensitivity directions are present.

\begin{remark}[Relation to the VC bound]
The VC result in Section~\ref{sec:hypotheses} gives a distribution-free
uniform bound for attackability indicators. The elementary lifting in
Proposition~\ref{prop:vc-affine} gives an $O(d^2)$ complexity, and the
polynomial-threshold improvement in Proposition~\ref{prop:vc-improved}
sharpens this to $O(d)$. The margin bounds above use additional
regularity, namely bounded readout norm and bounded local anisotropy, or
trace-sensitive control of the local sensitivity matrices, to obtain a
scale-sensitive statement in terms of the empirical distribution of the
adjusted margins. This is closer to classical margin-based learning theory
and is likely to be the more informative bound when most training examples
have large adjusted margin.
\end{remark}

The deterministic anisotropy bound in Proposition~\ref{prop:cover-margin} can
be relaxed to a high-probability bound. This connects the tail perspective of
Theorem~\ref{thm:Aeta-margin} with the uniform adjusted-margin analysis above:
high-anisotropy points are charged to a tail term, while the complexity term is
computed using an anisotropy quantile.

\begin{corollary}[Quantile-relaxed adjusted-margin bound]
\label{cor:quantile-margin}
Assume $\|z(x)\|_2\leq R$ for all $x$ and let
$\Theta_{W,B_0}=\{(w,b):\|w\|_2\leq W,
|b|\leq B_0\}$. Fix $\beta\in[0,1)$ and choose
$\Lambda_\beta>0$ such that
\[
\mathbb P\{\|S(X)\|_{\mathrm{op}}>\Lambda_\beta\}\leq\beta.
\]
Let
\[
M_{\eta,\beta}:=WR+B_0+\eta W\sqrt{\Lambda_\beta}.
\]
Then there is a universal constant $C$ such that, for every
$\rho\in(0,M_{\eta,\beta}]$, with probability at least $1-\delta$,
uniformly over $\theta\in\Theta_{W,B_0}$,
\begin{align}\label{eq:quantile-margin-bound}
R_\eta(\theta)
&\leq
\beta+\widehat R_{n,\eta,\rho}(\theta) \\
&\quad + C\sqrt{\frac{
 d\log\left(1+\frac{C W(R+\eta\sqrt{\Lambda_\beta})}{\rho}\right)
 +\log\left(1+\frac{C B_0}{\rho}\right)
 +\log(1/\delta)}{n}}.
\end{align}
\end{corollary}

The price for avoiding a deterministic uniform bound on $\|S(x)\|_{\mathrm{op}}$
is the additional tail term $\beta$. The empirical margin term remains the
ordinary adjusted-margin term from \eqref{eq:emp-margin-risk}, computed with
the actual $S(x)$ on the observed sample, while the complexity term depends
only on the quantile level $\Lambda_\beta$. Since
$\|S(x)\|_{\mathrm{op}}=\lambda^*(x)$, this is the margin-bound analogue of the
quantile relaxation in Theorem~\ref{thm:Aeta-margin}.

\section{Finite Paraphrase Search}\label{sec:Finite Paraphrase}

The preceding sections analyse an idealised local adversary that may choose
any point in a proxy-defined ellipsoid. Practical text adversarial attacks
are not of this form. They search over a finite set of candidate paraphrases
$\mathcal P_K(x)=\{x'_1,\ldots,x'_K\}$, obtained from a paraphrase generator,
beam search over substitutions, or gradient-based decoding of soft-token
iterates \citep{guo2021gradient,jang2016categorical}. Each candidate is
either proxy-valid (within distance $\eta$ of $x$) or not, and the attack
inspects them one by one. There is no continuous adversary in practice;
there is a discrete list of $K$ candidates and a decision on each.

This section bridges the two settings. The continuous theory in the
preceding sections gives a geometric characterisation of worst-case
proxy-bounded displacement and a statistical theory of attackability. The
finite-search setting is closer to what empirical work actually does.
Connecting them requires care because success and failure of finite search
have asymmetric implications. The connection is established inside a
linearised local chart around the base text, and the quantitative results
below are local approximations rather than global geometric facts about
the embeddings.

\subsection{Finite-search analogues of the continuous quantities}

\begin{definition}[Search procedure]\label{def:search-procedure}
A \emph{search procedure} of size $K$ is a, possibly randomised, map
$\mathcal P_K:\mathcal X\to 2^{\mathcal X}$ that assigns to each base text
$x\in\mathcal X$ a candidate set $\mathcal P_K(x)\subseteq\mathcal X$ of
cardinality at most $K$. The procedure may depend on the readout $(w,b)$
and the budget $\eta$; this dependence is suppressed in the notation.
Practical realisations include paraphrase generators, beam search over
substitution candidates, and gradient-based decoding of soft-token
iterates~\citep{guo2021gradient,jang2016categorical}. By abuse of notation,
$\mathcal P_K(x)$ refers either to the procedure or to its output at $x$.
\end{definition}

Fix a search procedure $\mathcal P_K$ and write
$\mathcal P_K(x)=\{x'_1,\ldots,x'_K\}$ for its output at $x$, with the
convention that some entries may be omitted when the procedure returns
fewer than $K$ candidates. Define the proxy-valid subset
\[
\mathcal P_{K,\eta}(x):=
\{x'\in\mathcal P_K(x):d_P(x,x')\leq \eta\},
\]
where $d_P(x,x'):=\|e_P(x)-e_P(x')\|_2$ is the proxy distance. The
finite-search analogue of the worst-case representation displacement at
budget $\eta$ is
\[
\Delta^{\mathrm{fin}}_{M,K,\eta}(x):=
\max\Bigl(\{0\}\cup
\{\|e_M(x')-e_M(x)\|_2:x'\in\mathcal P_{K,\eta}(x)\}\Bigr).
\]
The convention $\max\{0\}$ handles the case
$\mathcal P_{K,\eta}(x)=\emptyset$, where no candidate is proxy-valid and
no finite displacement is achieved.

For a fixed affine readout $(w,b)$, write
\[
s_{w,b}(x)=w^\top e_M(x)+b,\qquad
\sigma_x=\operatorname{sign}(s_{w,b}(x)),
\]
assuming $s_{w,b}(x)\neq 0$. The finite-search displacement toward the
decision boundary is
\[
D^{\mathrm{fin}}_{K,\eta}(x;w,b):=
\max\Bigl(\{0\}\cup
\{-\sigma_x w^\top(e_M(x')-e_M(x)):x'\in\mathcal P_{K,\eta}(x)\}\Bigr).
\]
A finite candidate reaches the decision boundary whenever
$D^{\mathrm{fin}}_{K,\eta}(x;w,b)\geq\gamma_w(x)$, with
$\gamma_w(x)=|s_{w,b}(x)|$ the geometric margin. This is the
finite-search analogue of the linearised flip condition
$\eta\sqrt{\Sigma_w(x)}>\gamma_w(x)$ from
Theorem~\ref{thm:flip}.

\subsection{The asymmetry between finite and continuous attacks}

Finite-search success and finite-search failure carry asymmetric weight.
Success, $D^{\mathrm{fin}}_{K,\eta}(x;w,b)\geq\gamma_w(x)$, exhibits an
explicit proxy-valid paraphrase that crosses the decision boundary, and
is a direct certificate of non-robustness. Failure,
$D^{\mathrm{fin}}_{K,\eta}(x;w,b)<\gamma_w(x)$, only certifies that the
particular set $\mathcal P_K(x)$ contains no flip; paraphrases outside
the candidate set may still be proxy-valid and may flip the prediction,
and $\lambda^*(x)$ may be arbitrarily large.

The corresponding empirical quantity is therefore the finite-search
attackability
\[
\mathcal A^{\mathrm{fin}}_w(\eta;\mathcal P_K)
:=\mathbb P\{D^{\mathrm{fin}}_{K,\eta}(X;w,b)\geq\gamma_w(X)\}
\leq\mathcal A_w(\eta),
\]
where the inequality holds up to the second-order linearisation error
inherent in passing from the exact finite-search displacement
$D^{\mathrm{fin}}$ to its linearised analogue $D^{\mathrm{lin}}$ (cf.\
Theorem~\ref{thm:linearization-error}). A finite-search success at $x$
exhibits an exact proxy-valid paraphrase flip, while $\mathcal A_w(\eta)$
is the population attackability of the linearised local model, so the
two sides are aligned only up to the linearisation residual. A confidence bound on
$\mathcal A^{\mathrm{fin}}_w$ follows by inverting the binomial likelihood:
if zero attacks are found in $n$ independent trials then, with confidence
$1-\delta$, $\mathcal A^{\mathrm{fin}}_w(\eta;\mathcal P_K)\leq 1-\delta^{1/n}
\leq\log(1/\delta)/n$. This bounds the procedure, not the continuous
adversary; closing that gap is the role of the covering conditions in the
next subsection.

\subsection{A covering condition makes finite search faithful}

The gap between continuous and finite attacks closes when the candidate set
$\mathcal P_K(x)$ is dense enough to approximate the continuous adversary.
The relevant notion is a covering radius of the proxy ball in the local
chart.

By Lemma~\ref{lem:localrep}, paraphrases of $x$ are parameterised locally by a
coordinate $u$ in some $q$-dimensional neighbourhood $U\subseteq\mathbb R^q$,
with the proxy and target embeddings admitting linear approximations
$E_P(u)-E_P(0)\approx J_P(x)u$ and $E_M(u)-E_M(0)\approx J_M(x)u$ at
$u=0$. Each candidate $x'_k\in\mathcal P_K(x)$ corresponds to some local
coordinate $u_k\in U$, and the candidate set defines a finite point cloud
$\{u_1,\ldots,u_K\}$ in local coordinates.

\begin{remark}[Local nature of the results]\label{rem:local-nature}
The bounds in this section are statements about the linearised local
chart. They are meaningful for proxy budgets $\eta$ inside the radius of
validity of $E_P(u)-E_P(0)\approx J_P(x)u$ and $E_M(u)-E_M(0)\approx
J_M(x)u$; the unmodelled error is $O(\|u\|^2)$. At $\eta$ much larger
than this scale, the bounds have no immediate content for the original
embedding geometry.
\end{remark}

The propositions below operate on the linearised analogue of
$D^{\mathrm{fin}}_{K,\eta}(x;w,b)$, namely
\[
D^{\mathrm{lin}}_{K,\eta}(x;w,b):=
\max\bigl(\{0\}\cup
\{-\sigma_x w^\top J_M(x)u_k:1\leq k\leq K,\|u_k\|_{B(x)}\leq\eta\}\bigr).
\]
Here $-\sigma_x w^\top(e_M(x')-e_M(x))$ is replaced by its first-order
approximation $-\sigma_x w^\top J_M(x)u_k$, and the proxy-validity test
$d_P(x,x')\leq\eta$ is replaced by $\|u_k\|_{B(x)}\leq\eta$. Both
replacements incur second-order error in $\|u_k\|$, so within the radius
of validity of Remark~\ref{rem:local-nature},
\[
D^{\mathrm{lin}}_{K,\eta}(x;w,b)
=
D^{\mathrm{fin}}_{K,\eta}(x;w,b)+O(\|u\|^2).
\]

\begin{proposition}[Finite search under a local covering condition]
\label{prop:finite-search-cover}
Fix a base text $x$ and assume the conditions of Lemma~\ref{lem:localrep} with
$B(x)\succ 0$. Let $\mathcal B_x(\eta):=\{u\in U:u^\top B(x)u\leq\eta^2\}$
be the linearised proxy ball at budget $\eta$. Suppose $\delta<\eta$ and
that the candidate set $\mathcal P_K(x)$, in local coordinates
$\{u_1,\ldots,u_K\}$, forms a $\delta$-cover of $\mathcal B_x(\eta-\delta)$
with respect to the proxy metric: for every $u\in\mathcal B_x(\eta-\delta)$
there is some $u_k$ with $\|u-u_k\|_{B(x)}\leq\delta$, where
$\|v\|_{B(x)}:=\sqrt{v^\top B(x)v}$. Then for any affine readout $(w,b)$,
\[
(\eta-2\delta)\sqrt{\Sigma_w(x)}
\leq D^{\mathrm{lin}}_{K,\eta}(x;w,b)
\leq \eta\sqrt{\Sigma_w(x)}.
\]
In particular, if $\delta\leq\eta\epsilon$ for some $\epsilon\in(0,1/2)$, then
\[
D^{\mathrm{lin}}_{K,\eta}(x;w,b)
\geq (1-2\epsilon)\eta\sqrt{\Sigma_w(x)}.
\]
\end{proposition}

The upper bound is automatic: every candidate with $\|u_k\|_{B(x)}\leq\eta$
has linearised readout displacement at most the continuous worst-case
displacement $\eta\sqrt{\Sigma_w(x)}$. The lower bound is the substantive
content. It says that if the candidate set $\delta$-covers the slightly
smaller ball $\mathcal B_x(\eta-\delta)$, then the linearised finite-search
displacement comes within $2\delta\sqrt{\Sigma_w(x)}$ of the continuous
worst case. When $\delta$ is small (the candidate set is dense), the gap
is small, and finite search faithfully approximates the continuous
adversary up to the second-order error linking $D^{\mathrm{lin}}$ to
$D^{\mathrm{fin}}$.

A direct consequence is a quantitative version of the failure-to-certify
asymmetry: a candidate set that is a $\delta$-cover of $\mathcal B_x(\eta-\delta)$
with $\delta\leq\eta\epsilon$ turns failure of finite search into approximate
robustness against the continuous adversary.

\begin{corollary}[Robustness certification under a covering condition]
\label{cor:finite-search-robustness}
Under the assumptions of Proposition~\ref{prop:finite-search-cover}, suppose
that the candidate set $\mathcal P_K(x)$ is a $\delta$-cover of
$\mathcal B_x(\eta-\delta)$ in the proxy metric with $\delta\leq\eta\epsilon$
for some $\epsilon\in(0,1/2)$. If
$D^{\mathrm{lin}}_{K,\eta}(x;w,b)<\gamma_w(x)$, then
\[
(1-2\epsilon)\eta\sqrt{\Sigma_w(x)}<\gamma_w(x),
\]
equivalently
\[
\eta\sqrt{\Sigma_w(x)}<\frac{\gamma_w(x)}{1-2\epsilon}.
\]
\end{corollary}

Failure of finite search in the linearised chart, when the candidate set is
a $\delta$-cover of $\mathcal B_x(\eta-\delta)$ with $\delta\leq\eta\epsilon$,
certifies robustness against the continuous adversary at budget $\eta(1-2\epsilon)$
rather than the original budget $\eta$. As the cover becomes denser
($\epsilon\to 0$), the certified budget approaches $\eta$ and the
certificate becomes tight.

\subsection{Beyond covering: probabilistic candidate sets}

Many empirical paraphrase generators do not aim for deterministic covering.
They sample candidates from some distribution over proxy-valid paraphrases,
either explicitly (sampling from a paraphrase model) or implicitly (running
a randomised optimisation procedure). For such procedures, a more natural
condition is that the candidate set covers the proxy ball with high
probability, where the probability is taken over the sampling distribution.

\begin{proposition}[Probabilistic covering]\label{prop:prob-covering}
Suppose the candidate coordinates $u_1,\ldots,u_K$ are drawn i.i.d.\ from a
distribution on the local chart with density $f$ satisfying $f(u)\geq c>0$
for all $u\in\mathcal B_x(\eta)$, and set $\kappa:=c/\sqrt{\det B(x)}$.
 Then for every $\delta\in(0,\eta)$,
\[
\mathbb P\{\mathcal P_K(x)\text{ is a $\delta$-cover of }\mathcal B_x(\eta-\delta)\}
\geq 1-N(\delta/2,\mathcal B_x(\eta-\delta))(1-\kappa v_q(\delta/2)^q)^K,
\]
where $v_q$ is the volume of the unit Euclidean ball in $\mathbb R^q$ and
$N(\delta/2,\mathcal B_x(\eta-\delta))$ is the $\delta/2$-covering number of
$\mathcal B_x(\eta-\delta)$ in the proxy metric.
\end{proposition}

\begin{remark}[Applicability of the density-floor hypothesis]\label{rem:density-floor}
The density-floor condition $f\geq c$ requires the sampling distribution to be
full-dimensional on the local chart. It covers uniform sampling on
$\mathcal B_x(\eta)$ and noise-injection or randomised-optimisation
schemes with continuous densities, but excludes samplers whose support
is contained in a Lebesgue-null subset, such as discrete paraphrase
samplers or one-dimensional samplers along a fixed direction.
Proposition~\ref{prop:gradient-search} gives the analogue for the
one-dimensional case; discrete samplers fall outside the reach of both
results and can only be analysed via the deterministic
Proposition~\ref{prop:finite-search-cover} or empirically.
\end{remark}

In particular, $K=\Theta(\delta^{-q}\log(1/\delta))$ candidates suffice to
$\delta$-cover $\mathcal B_x(\eta-\delta)$ with high probability. This is
exponential in the local chart dimension $q$, the standard difficulty of
dense sampling in high-dimensional spaces. Practical attacks circumvent
this difficulty by exploiting structure in the embeddings: they bias the
sampling toward directions where the readout score moves quickly, using
gradient information \citep{guo2021gradient}, or they restrict to discrete
paraphrases that are linguistically plausible
\citep{alzantot2018generating,jin2020bert}.

\subsection{Readout-aligned search}

A practical attack that uses information about the readout can achieve much
better effective coverage than uniform sampling. The basic move is to
search not over the whole proxy ball but along the readout-specific
worst-case direction of the linearised attacker, the closed-form solution
of Theorem~\ref{thm:flip}.

\begin{proposition}[Effective one-dimensional coverage]
\label{prop:gradient-search}
Let $(w,b)$ be a fixed affine readout with $\Sigma_w(x)>0$ and define the
linearised worst-case direction
\[
r_w(x):=-\sigma_x\,\frac{B(x)^{-1}J_M(x)^\top w}{\sqrt{\Sigma_w(x)}},
\]
which satisfies $\|r_w(x)\|_{B(x)}=1$. Suppose the candidate coordinates
$\{u_1,\ldots,u_K\}$ contain points $\{\alpha_j r_w(x):j=1,\ldots,J\}$ for
some scalars $\alpha_j\in[0,\eta]$ that $\delta'$-cover the interval
$[0,\eta]$. Then
\[
D^{\mathrm{lin}}_{K,\eta}(x;w,b)
\geq
(\eta-\delta')\sqrt{\Sigma_w(x)}.
\]
\end{proposition}

The proposition is the readout-specific analogue of
Proposition~\ref{prop:finite-search-cover} along a single direction: a
$J$-point cover of $[0,\eta]$ along $r_w(x)$ recovers the continuous
worst case up to $\delta'\sqrt{\Sigma_w(x)}$, without requiring a
$q$-dimensional cover of the whole proxy ball. The direction $r_w(x)$ is
not the leading generalised eigenvector $v^*(x)$ of $(A(x),B(x))$:
$v^*(x)$ maximises target-representation displacement in the proxy metric
and is the right object for $\lambda^*(x)$, but for flipping a fixed
affine readout it is the readout-direction displacement that matters, and
the direction that maximises it is $r_w(x)$.

The limitation of one-dimensional readout-aligned search is that it
exploits a single readout. If the empirical attack is to certify
robustness over a class of readouts, the alignment direction has to be
recomputed for each readout, and the bound applies only to the readouts
covered.

\subsection{Empirical interpretation}

The section's overall message for empirical work on text adversarial
robustness can be summarised in three principles.

First, successful attacks are unambiguous evidence of vulnerability and
require no theoretical bridge. A paraphrase that is proxy-valid and flips
the prediction is a witness.

Second, failed attacks should be reported alongside the search procedure
that produced them, not as evidence of robustness against the continuous
adversary. The relevant population quantity for empirical evaluation is
the finite attackability $\mathcal A^{\mathrm{fin}}_w(\eta;\mathcal P_K)$, not
the continuous $\mathcal A_w(\eta)$.

Third, the gap between the two attackabilities is controlled by the quality
of the search procedure. Proposition~\ref{prop:finite-search-cover} gives a
deterministic sufficient condition (proxy-metric covering),
Proposition~\ref{prop:prob-covering} a probabilistic one (density-bounded
random sampling), and Proposition~\ref{prop:gradient-search} a structural
one (one-dimensional search along the readout-specific direction
$r_w(x)$). Each gives a different way to certify, within the linearised
local chart, that finite search faithfully approximates the continuous
adversary.

A practical attack that targets the readout direction $r_w(x)$, samples
densely along the corresponding interval, and reports both successful and
unsuccessful flips gives an empirical certificate that the continuous
theory can interpret.
Attacks that do not specify their search structure produce numbers that
are difficult to compare across studies and that the continuous theory
cannot underwrite.

\section{Empirical verification}\label{sec:experiments}

This section examines how the geometric quantities of the preceding sections
behave on a deployed financial sentiment classifier. The theory rests on a
continuous local relaxation of a discrete paraphrase space, and its
guarantees are stated in terms of the Jacobians $J_M$ and $J_P$, the
attackability index $\lambda^*(x)$, and the adjusted margin.

We report five experiments, instantiated across two local charts. Experiment~1
computes $\lambda^*(x)$ and the readout-direction sensitivity $\Sigma_w(x)$
across inputs and tests the inequality $\Sigma_w(x)\le\lambda^*(x)$ of
Theorem~\ref{thm:sigma-lambda}; it serves both as a numerical check on the
Jacobian and generalised-eigenvalue pipeline and as a measurement of the slack
ratio $\lambda^*(x)/\Sigma_w(x)$, which controls the looseness of the
readout-free certificate of Corollary~\ref{cor:worstcase-flip} and of the
margin-tail bound of Theorem~\ref{thm:Aeta-margin}. Experiment~2 compares the
linearised flip threshold $\eta=Z_w(x)$ of Theorem~\ref{thm:flip} against the
realised nonlinear displacement along the same worst-case direction,
calibrating the budget range over which the first-order model of
Theorem~\ref{thm:linearization-error} is informative. Experiment~3 estimates the
empirical attackability curve $\widehat A_{n,w}(\eta)$ and the deterministic
empirical-inclusion overlays implied by Theorem~\ref{thm:Aeta-margin} and
Remark~\ref{rem:Aeta-sigma}, and compares it with the concentration rate of
Theorem~\ref{thm:Aeta-empirical}. Experiment~4 evaluates finite-search
attackability against the candidate budget $K$, together with a
one-dimensional coverage diagnostic and a best-of-all-candidates oracle, making
explicit the failure-to-certify asymmetry of Section~\ref{sec:Finite Paraphrase}.
Experiment~5 tests whether the adjusted margin $Z_w(x)$ predicts realised
paraphrase flips. Experiments~1 and~3 are run in both charts; Experiment~2 in
the soft-token chart only; Experiments~4 and~5 in the paraphrase-cloud chart.

The theory is built on continuous local charts, but text is discrete. We
therefore evaluate two procedures for instantiating these charts from the
paraphrase space $\mathcal{X}$. In Section~\ref{sec:soft-token-relaxation}, $x\in\mathcal{X}$ is represented
as a sequence of one-hot token vectors and relaxed to soft distributions over
the vocabulary, a continuous relaxation of discrete text inputs underlying
gradient-based text attacks \citep{guo2021gradient,jang2016categorical}; the
latent coordinate $u\in U$ parametrises perturbations of those soft tokens,
and the Jacobians are computed by automatic differentiation
\citep{baydin2018automatic} rather than estimated. In
Section~\ref{sec:paraphrase-cloud}, we generate real paraphrases of $x$, embed
each under $M$ and $P$, and fit a local linear model, the finite-search
regime of Section~\ref{sec:Finite Paraphrase}.
Section~\ref{sec:exp-setup} fixes the shared setup.

\subsection{Experimental setup}\label{sec:exp-setup}

The target model $M$ is FinBERT \citep{araci2019finbert}, a
financial sentiment classifier, whose head is a single affine layer on its
pooled representation. The local results of
Sections~\ref{sec:Prediction-Flip}--\ref{sec:margins} depend on the target
map $e_M$ only through its Jacobian $J_M$, so we take $e_M(x)$ to be the raw
pooled representation rather than its $\ell_2$-normalisation; the readout
$(w,b)$ is then exactly the deployed FinBERT head, read off the model rather
than imposed. The proxy model $P$ is the all-MiniLM-L6-v2 Sentence-BERT
embedder \citep{reimers2019sentence}, whose output is mean-pooled and
$\ell_2$-normalised in the standard way and defines the semantic budget
through the proxy distance $d_P$. Since $e_P$ is unit-normalised,
$\|e_P(x)-e_P(x')\|_2^2 = 2 - 2\cos(e_P(x),e_P(x'))$, so the proxy
ball $\{x': d_P(x,x')\le \eta\}$ coincides with the cosine threshold
$\{x': \cos(e_P(x),e_P(x'))\ge 1 - \eta^2/2\}$. Both models share FinBERT's WordPiece vocabulary, so the soft-token chart of Section~\ref{sec:soft-token-relaxation} can parametrise a single latent coordinate feeding both embedding maps, a construction that is well defined only because the two models tokenise text over a common vocabulary.  The embedding
dimensions are $d_M=768$ and $d_P=384$ and so, as the theory permits, do not
coincide.

Inputs are drawn from the all-agreement subset of the Financial
PhraseBank \citep{malo2014good}, on which all annotators assign the
same sentiment label. The three sentiment classes of ``Positive", ``Negative" and ``Neutral", are handled
through a logit-gap construction. At a labelled example with true
class $t$, we apply the binary affine analysis of
Sections~\ref{sec:Prediction-Flip}--\ref{sec:margins} to the
effective readout $w := w_t - w_{i^\star}$ and $b := b_t - b_{i^\star}$,
where $\{(w_i,b_i)\}_{i=1}^{3}$ are the rows of the deployed FinBERT
classification head and
\[
i^\star \in \arg\min_{i\neq t}\,\bigl[(w_t-w_i)^\top e_M(x)+(b_t-b_i)\bigr]
\]
is the closest competing class. The induced score
$w^\top e_M(x)+b$ is then the smallest gap between the true-class
logit and any competing-class logit, and a sign change in this score
is a prediction flip in the multiclass classifier; logit-gap
objectives of this form are standard in adversarial-attack design
\citep{carlini2017towards}. We sample 200 sentences uniformly without replacement from the all-agreement subset and retain the $n=190$ on which FinBERT's hard prediction matches the labelled sentiment, so that a realised
paraphrase flip is unambiguously a flip away from the base prediction. The empirical attackability
curve then concentrates around its population counterpart at the Dvoretzky--Kiefer--Wolfowitz rate
of Theorem~\ref{thm:Aeta-empirical}, fixing the band half-width in Figure~\ref{fig:attackability_curve} at $\sqrt{\ln(2/\delta)/2n}=1.36/\sqrt n \approx 0.099$ for $\delta=0.05$. The paraphrase-cloud construction of Section~\ref{sec:paraphrase-cloud} is evaluated on the $n'=185$ inputs for which the cloud displacement pencil is non-degenerate and at least one paraphrase survives filtering, giving a half-width $1.36/\sqrt{n'}\approx 0.100$ in Figure~\ref{fig:attackability_curve_cloud}.

The proxy budget $\eta$ is swept over a grid and reported as a fraction of
the median proxy distance $d_P(x,x')$ between an input and its generated
paraphrases, so that $\eta$ is interpretable on the scale of genuine
paraphrasing rather than in raw embedding units. Proxy-null directions are
regularised by replacing $B$ with $B+\nu I$, as discussed
after~\eqref{eq:local-problem} in Section~\ref{sec:Local Semantic}, with
$\nu=10^{-3}\operatorname{tr}(B)/q$; we verify that reported quantities are
stable as $\nu$ ranges over $[10^{-4},10^{-2}]\operatorname{tr}(B)/q$. Because $\nu I$ is referred to the coordinate frame, the regularized pencil $(A,B+\nu I)$ is not reparameterisation-invariant, so the reported $\lambda^*(x)$ and $\Sigma_w(x)$ are coordinate-dependent approximations to the intrinsic quantities of Proposition~\ref{prop:chart-inv}. They are recovered as $\nu\to 0$ on non-degenerate inputs, and their stability across the stated $\nu$ range indicates that the approximation does not affect the conclusions.
Confidence bands use $\delta=0.05$. Jacobians are obtained by automatic
differentiation in PyTorch, and the generalised eigenproblems for the pencil
$(A,B)$ in~\eqref{eq:lambdaAB} are solved with a symmetric-definite solver
applied to $(A ,B+\nu I)$. The local coordinate dimension $q$ is
method-specific and is fixed in Sections~\ref{sec:soft-token-relaxation}
and~\ref{sec:paraphrase-cloud}.

\subsection{Soft-token relaxation}\label{sec:soft-token-relaxation}

The chart instantiates the local model of Section~\ref{sec:Local Semantic}
via a continuous relaxation of the discrete token sequence in the spirit of
Gumbel--softmax adversarial decoding
\citep{jang2016categorical,guo2021gradient}. GBDA samples from a Gumbel--softmax distribution to support
stochastic gradient updates on the parameter matrix; we use the
deterministic softmax at the basepoint $u=0$ because our use of the
relaxation is to compute Jacobians of $E_M, E_P$, not to sample
adversarial candidates. Let $\mathcal V$ denote the common WordPiece vocabulary, with cardinality
$|\mathcal V|$. For an input of length $L$
WordPiece tokens, let $\ell^{(0)}\in\mathbb{R}^{L|\mathcal V|}$ be a stack of
position-wise logit vectors initialised so that $\softmax(\ell_i^{(0)})$ is a
sharp approximation to the one-hot encoding of the original token at
position $i$; concretely we set the logit of the original token to $\tau=20$ and all other entries to zero,
which places mass $e^{20}/(e^{20}+|V|-1)\approx 0.99994$ on the original token over FinBERT's WordPiece
vocabulary ($|V|=30,522$), so the basepoint faithfully approximates the one-hot encoding.
Each model $m\in\{M,P\}$ ingests the soft-token sequence through its own
WordPiece embedding $E^{(m)}_{\mathrm{tok}}:\mathcal V\to\mathbb{R}^{h_m}$,
\[
\widetilde T^{(m)}_i(\ell)=\sum_{v\in \mathcal V}\softmax(\ell_i)_vE^{(m)}_{\mathrm{tok}}(v).
\]
The soft distribution $\softmax(\ell_i)$ is shared between the two models
because they tokenise over a common vocabulary; the embedding matrices that
follow are model-specific. The local coordinate
$u\in\mathbb{R}^q$ parametrises perturbations of $\ell$ via
\begin{equation}\label{eq:soft-token-chart}
\ell(u) = \ell^{(0)} + Gu,
\end{equation}
where $G\in\mathbb{R}^{L|\mathcal V|\times q}$
has orthonormalised columns and is drawn independently for each input, so the measured geometry
reflects directions adapted to each sentence rather than a single fixed random slice. The composition with the rest of $M$ and $P$
defines the smooth local maps $E_M(u)$ and $E_P(u)$ of
Lemma~\ref{lem:localrep}, and the Jacobians $J_M(0)$ and
$J_P(0)$ are obtained in closed form by forward-mode Jacobian--vector
products on the columns of $G$, requiring $q$ forward passes per input
through each model. We report results at $q\in\{32,64,128\}$ and verify
qualitative stability across this range; the figures use $q=32$.

Experiment 2 is the single exception to $\tau=20$. At the near-one-hot basepoint the softmax Jacobian
is vanishingly small, so $\beta=\lambda_{\min}(B+\nu I)$ collapses and the worst-case direction leaves the
chart's first-order regime at budgets far below any genuine paraphrase. We therefore evaluate the
linearisation diagnostic at $\tau=10$ ($\mathrm{mass}\approx 0.42$ on the original token), a smoother operating point
at which the first-order regime is resolvable; Experiments 1 and 3 retain $\tau=20$, at which $\lambda*$, $\Sigma_w$, and $Z_w$, being scale-invariant ratios, have converged.

\paragraph{Experiment 1 (soft-token chart): $\mathbf{\Sigma_w(x)\le\lambda^*(x)}$.}

For each input $x$ in the sample we form the $q\times q$ pencil
$(A,B+\nu I)$ from the soft-token Jacobians, solve the symmetric-definite
generalised eigenproblem to obtain $\lambda^*(x)$, and compute the
readout-direction sensitivity
$\Sigma_w(x)=w^\top J_M(B+\nu I)^{-1}J_M^\top w$ for the logit-gap readout
$(w,b)$ of Section~\ref{sec:exp-setup} by solving the $q\times q$
symmetric system $(B+\nu I)y=J_M^\top w$ and forming
$(J_M^\top w)^\top y$. Figure~\ref{fig:sigma-lambda} plots $\Sigma_w(x)$
against $\lambda^*(x)$ across the $n=190$ inputs together with the line
$\Sigma_w=\lambda^*$ that upper-bounds the scatter by
Theorem~\ref{thm:sigma-lambda}. The empirical distribution of the slack
ratio $\lambda^*(x)/\Sigma_w(x)$ is reported as a histogram in the same
figure: a slack ratio concentrated near one indicates that the deployed
readout is aligned with the worst direction of the local pencil, while a
larger slack indicates that worst-direction movement is largely orthogonal
to the readout.

\paragraph{Experiment 2 (soft-token chart): linearisation error.}
For each input we take the closed-form linearised worst direction toward
the decision boundary,
\[
u^\star_{\mathrm{flip}}(\eta;x) =
-\operatorname{sign}(s_0)\eta
\frac{(B+\nu I)^{-1}J_M^\top w}{\sqrt{\Sigma_w(x)}},
\]
of Theorem~\ref{thm:flip}, apply it through the \emph{nonlinear} $E_M$,
and record the actual readout displacement
\[
\Delta s^{\mathrm{nl}}(\eta;x) :=
w^\top\bigl(E_M(u^\star_{\mathrm{flip}})-E_M(0)\bigr).
\]
The linearised prediction at the same $u^\star_{\mathrm{flip}}$ is
$\Delta s^{\mathrm{lin}}(\eta;x)=-\operatorname{sign}(s_0)\,\eta\sqrt{\Sigma_w(x)}$.
Figure~\ref{fig:linearisation-error} plots, separately, the mean absolute nonlinear displacement $n^{-1}\sum_i|\Delta s^{nl}(\eta;x_i)|$ and the mean absolute linearised displacement $n^{-1}\sum_i|\Delta s^{lin}(\eta;x_i)|$, and, in a second panel, the
residual $|\Delta s^{nl}-\Delta s^{lin}|$ normalised per input by $\eta^2/\beta_i$ with $\beta_i=\lambda_{\min}(B+\nu I)$, since the
aggregate mean and median otherwise confound the $\eta$-scaling with heterogeneity in $1/\beta$. By
Theorem~\ref{thm:linearization-error} the normalised residual is flat in the regime where the first-order model is
informative; the experiment calibrates the budget range over which this holds, rather than
testing a pass-or-fail radius of validity.

\paragraph{Experiment 3 (soft-token chart): empirical attackability curve.}
The attackability margin
$Z_w(x_i)=\gamma_w(x_i)/\sqrt{\Sigma_w(x_i)}$ is computed in the
soft-token chart for each of the $n=190$ inputs.
Figure~\ref{fig:attackability_curve} shows the empirical strict left-CDF
\[
\hat A_{n,w}(\eta) = \frac{1}{n}\sum_{i=1}^{n}\mathbbm{1}\{Z_w(x_i)<\eta\}
\]
together with the Dvoretzky--Kiefer--Wolfowitz band of width
$1.36/\sqrt n$ at $\delta=0.05$ from
Theorem~\ref{thm:Aeta-empirical}. Two families of conservative quantile
bounds are overlaid, each for $\beta\in\{0.10,0.25\}$. The readout-free
bound of Theorem~\ref{thm:Aeta-margin} uses the empirical
$(1-\beta)$-quantile $\hat\Lambda_{1-\beta}$ of $\lambda^*(x_i)$ and is
labelled $\lambda^*$ in the figure; the readout-specific bound of
Remark~\ref{rem:Aeta-sigma} uses the empirical $(1-\beta)$-quantile
$\hat\Sigma_{1-\beta}$ of $\Sigma_w(x_i)$ and is labelled $\Sigma_w$.
These overlays are the deterministic empirical-inclusion bound, not a finite-sample population
certificate. Writing $P_n$ for the empirical measure and $\widehat{\Lambda}_{1-\beta}$ for the empirical $(1-\beta)$-quantile
of $\lambda^*(x_i)$,
\[
\widehat A_{n,w}(\eta)
=
P_n[\gamma_w < \eta\sqrt{\Sigma_w}]
\leq
P_n[\gamma_w < \eta\sqrt{\lambda^*}]
\leq
P_n[\gamma_w < \eta\sqrt{\widehat{\Lambda}_{1-\beta}}] + \beta,
\]
an exact in-sample inequality (the last step because $P_n[\lambda^* > \widehat{\Lambda}_{1-\beta}]\leq \beta$). The analogous $\Sigma_w$
overlay uses the empirical quantile of $\Sigma_w$; since $\Sigma_w \leq \lambda^*$ it lies below the $\lambda^*$ overlay at
matching $\beta$. The gap between each overlay and $\widehat A_{n,w}(\eta)$ is the empirical looseness of the
quantile relaxation.

\subsection{Paraphrase-cloud estimation}\label{sec:paraphrase-cloud}

The chart instantiates the local model from a finite set of generated
paraphrases rather than from a parameterised relaxation. For each input $x$,
using the attack of \citet{TURETKEN2026107698},
we generate up to $K$ paraphrases. After discarding near-duplicates and candidates
outside the proxy-similarity gate, the number retained varies per input, with a
median of 35 and a maximum of 40 (the generator's per-call cap). We form the target and proxy displacement matrices
\begin{equation}\label{eq:cloud-displacements}
D_M(x) := \bigl[e_M(x'_k)-e_M(x)\bigr]_{k=1}^{K}\in\mathbb{R}^{d_M\times K_x},
\quad
D_P(x) := \bigl[e_P(x'_k)-e_P(x)\bigr]_{k=1}^{K}\in\mathbb{R}^{d_P\times K_x}.
\end{equation}
The local coordinate is $u\in\mathbb{R}^{K_x}$, with the $k$\textsuperscript{th} unit basis
vector indexing the $k$\textsuperscript{th} retained paraphrase, and the local maps are linear
interpolations between $x$ and combinations of the candidates,
\begin{equation}\label{eq:cloud-chart}
E_M(u)-E_M(0) \approx D_M(x)\,u,
\qquad
E_P(u)-E_P(0) \approx D_P(x)\,u.
\end{equation}
The chart Jacobians are then $J_M(x)=D_M(x)$ and $J_P(x)=D_P(x)$, and the
pullback matrices $A=D_M^\top D_M$ and $B=D_P^\top D_P$ are $K\times K$;
the local chart dimension is $q=K_x$. The two charts of
Sections~\ref{sec:soft-token-relaxation} and~\ref{sec:paraphrase-cloud}
yield different Jacobians and therefore different numerical values of
$\lambda^*(x)$ and $\Sigma_w(x)$. By Proposition~\ref{prop:chart-inv}
each chart is internally invariant under reparametrisation, but the two
charts test the same structural predictions of
Sections~\ref{sec:Prediction-Flip}--\ref{sec:population} rather than
predicting the same scalars.

\paragraph{Experiment 1 (cloud chart): $\mathbf{\Sigma_w(x)\le\lambda^*(x)}$.}

We repeat Experiment~1 in the cloud chart, using the cloud-chart pencil $(D_M^\top D_M,\,D_P^\top D_P+\nu I)$.
Figure~\ref{fig:sigma-lambda-cloud} plots $\Sigma_w(x)$ against
$\lambda^*(x)$ across the $n'=185$ inputs. The comparison with
Figure~\ref{fig:sigma-lambda} is the substantive content: a slack ratio
concentrated more tightly near one in the cloud chart than in the
soft-token chart indicates that the paraphrase generator is biased
towards directions that move the readout, whereas a wider distribution
indicates that the generator explores directions that the worst-case
local geometry would not target.

\paragraph{Experiment 3 (cloud chart): attackability curve.}
We repeat Experiment~3 in the cloud chart and overlay
$\hat A_{n,w}(\eta)$, its DKW band, and the readout-free and
readout-specific quantile bounds of Theorem~\ref{thm:Aeta-margin} and
Remark~\ref{rem:Aeta-sigma} (labelled $\lambda^*$ and $\Sigma_w$, each
for $\beta\in\{0.10,0.25\}$) on
Figure~\ref{fig:attackability_curve_cloud}. Two qualitative outcomes are
diagnostic. If the cloud-chart $\hat A_{n,w}$ tracks the soft-token
$\hat A_{n,w}$ of Section~\ref{sec:soft-token-relaxation} within their
respective DKW bands, the two relaxations agree on the attackability of
the deployed classifier. If the two curves disagree, the disagreement
isolates the contribution of the paraphrase generator's inductive biases
relative to a model-agnostic local relaxation.

\paragraph{Experiment 4 (cloud chart): finite-search flip events.}

For each input we compute the linearised prediction
$\mathbbm{1}\{Z_w(x)<\eta\}$ in the full cloud chart. The
finite-search flip event
$\mathbbm{1}\{D^{\mathrm{fin}}_{K,\eta}(x;w,b)\ge\gamma_w(x)\}$ of
Section~\ref{sec:Finite Paraphrase} is then evaluated on subsampled
candidate sets $\mathcal P_K(x)\subseteq\mathcal P_{K_x}(x)$ for
$K\in\{4,8,12,18,25\}$, with each subsample drawn uniformly without
replacement.
Figure~\ref{fig:finite-search} reports the finite-search attackability
\[
A^{\mathrm{fin}}_w(\eta;\mathcal P_K) :=
\frac{1}{n}\sum_{i=1}^{n}\mathbbm{1}\{D^{\mathrm{fin}}_{K,\eta}(x_i;w,b)\ge\gamma_w(x_i)\}
\]
as a function of $\eta$ for each $K$. We report $A^{\mathrm{fin}}_w(\eta; P_K)$ for $K \in \{4, 8, 12, 18, 25\}$ and a best-of-all-candidates oracle that uses each input's full retained set, together with a one-dimensional coverage diagnostic: the median
over inputs of the largest gap, normalised by $\eta$, between consecutive proxy-valid candidate
positions along the readout-aligned direction $r_w(x)$ of Proposition~\ref{prop:gradient-search}. Subsampling cannot
establish a proxy-ball cover, i.e., the candidate count for a fixed-radius cover of a $q$-dimensional
ball grows exponentially in $q$, so the finite-search curve is not expected to reach $\widehat A_{n,w}$, and
we present the result as the failure-to-certify asymmetry of Section~\ref{sec:Finite Paraphrase}, a found attack is evidence
of vulnerability; a failed finite search is weak evidence of robustness, rather than as
confirmation of Proposition~\ref{prop:finite-search-cover}. Because $\widehat A_{n,w}$ is the linearised curve and over-states realised flips (Experiment~5), the gap between it and the oracle reflects both finite coverage and
first-order over-prediction; the coverage diagnostic isolates the former.

\paragraph{Experiment 5 (cloud chart): predictive validity of $Z_w$.} 
The decisive test of the term
attackability index is whether $Z_w(x)$ predicts realised flips. For each cloud input we record
the realised FinBERT prediction of every paraphrase and define the smallest realised flip budget
as the minimum proxy distance over paraphrases whose prediction differs from the base label
($\infty$ if none flips). Since smaller $Z_w$ means more attackable, we report the Spearman rank
correlation between $Z_w(x)$ and the realised flip budget over inputs that flip, and the AUC of
$-Z_w(x)$ for predicting whether a proxy-valid flip exists within budget $\eta$.

\subsection{Results}

We first consider Figure~\ref{fig:sigma-lambda}. Across all inputs and in both charts, every
$(\lambda^*(x),\Sigma_w(x))$ pair lies strictly below the diagonal
$\Sigma_w=\lambda^*$, i.e., zero violations of $\Sigma_w\leq\lambda^*$ in
the 190 soft-token inputs,  confirming Theorem~\ref{thm:sigma-lambda} pointwise under the
unit-norm readout. The slack is substantial: $\log_{10}(\lambda^*(x)/\Sigma_w(x))$
has median $\approx 2.05$ in the soft-token chart and $\approx 1.47$ in the
paraphrase-cloud chart, so the worst local direction typically expands the
target representation one-and-a-half to two orders of magnitude more than the
deployed readout does. The cloud-chart slack is smaller and more concentrated,
indicating that the paraphrase generator is mildly biased toward readout-moving
directions relative to the model-agnostic soft-token relaxation, though in
neither chart is the ratio close to one.

\begin{figure}[htbp]
    \centering
    \begin{subfigure}{\textwidth}
        \centering
        \includegraphics[width=\textwidth]{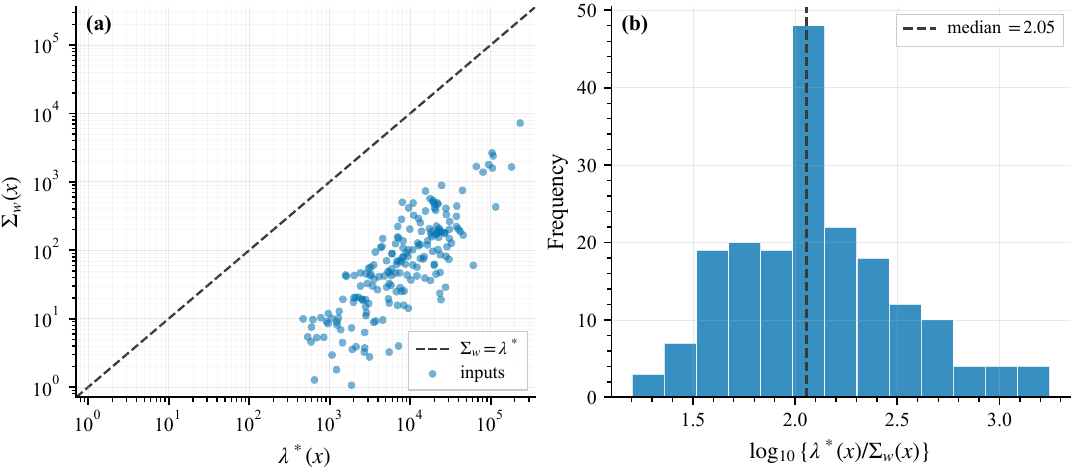}
        \caption{Soft-token relaxation}
        \label{fig:sigma-lambda-soft}
    \end{subfigure}
    \hfill
    \begin{subfigure}{\textwidth}
        \centering
        \includegraphics[width=\textwidth]{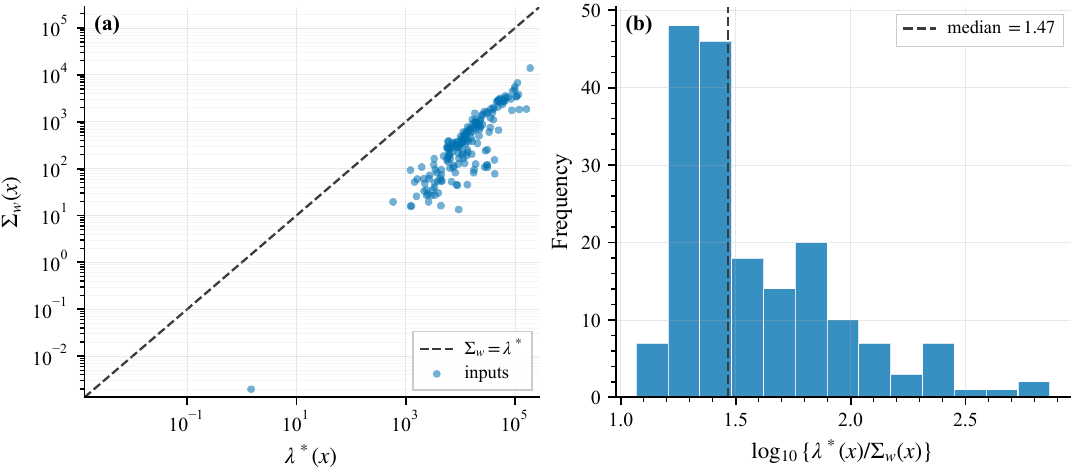}  
        \caption{Paraphrase cloud}
        \label{fig:sigma-lambda-cloud}
    \end{subfigure}
    \caption{$\Sigma_w(x)$ versus $\lambda^*(x)$ across inputs (log--log), with the
line $\Sigma_w=\lambda^*$ that bounds the scatter by
Theorem~\ref{thm:sigma-lambda}, and the empirical distribution of
$\log_{10}\{\lambda^*(x)/\Sigma_w(x)\}$. Soft-token chart $n=190$; paraphrase
cloud $n'=185$.}
    \label{fig:sigma-lambda}
\end{figure}

For the linearisation diagnostic, Figure~\ref{fig:linearisation-error}, the mean nonlinear displacement
tracks the linear displacement closely up to $\eta\approx 3\times 10^{-2}$ and
then peels below it as the soft-token embedding saturates within its convex hull
while the linear term keeps growing; the break is saturation of $\Delta s^{\mathrm{nl}}$, not Taylor blow-up. The per-input normalised residual confirms this: its median is flat across the small-budget range, the regime
in which the first-order model of Theorem~\ref{thm:linearization-error} is informative, and rises only
as saturation sets in, while the mean is inflated by a right-skewed minority of
inputs that leave the first-order regime early.

\begin{figure}[hbtp]
\centering
\includegraphics[width=\textwidth]{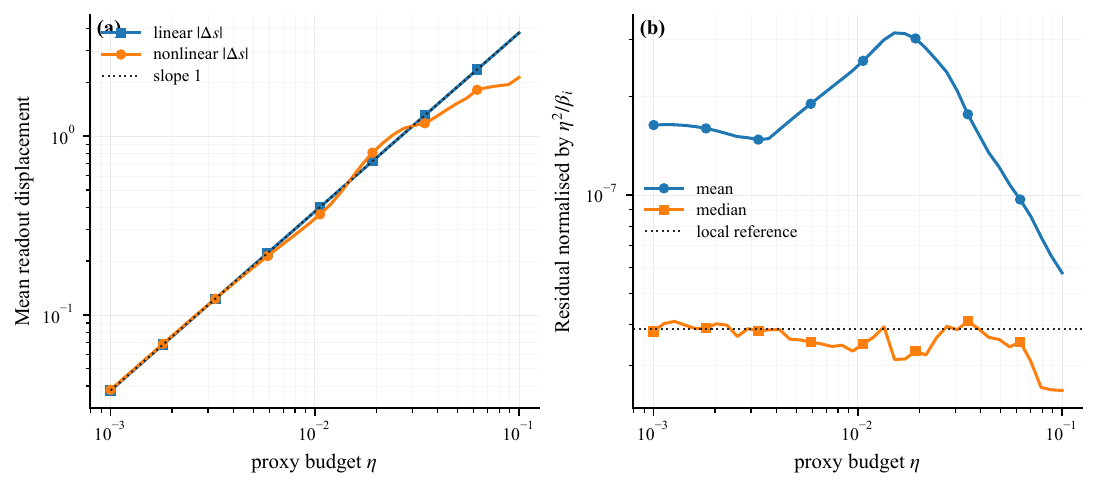}
\caption{Local linearisation diagnostic in the soft-token chart ($\tau=10$).
(a)~mean absolute linear and nonlinear readout displacement against the proxy
budget $\eta$ (log--log); (b)~the residual
$|\Delta s^{\mathrm{nl}}-\Delta s^{\mathrm{lin}}|$ normalised per input by
$\eta^2/\beta_i$, $\beta_i=\lambda_{\min}(B+\nu I)$. The flat median in~(b) marks
the budget range over which the first-order model is informative.}
\label{fig:linearisation-error}
\end{figure}

Figure~\ref{fig:attackability_curves} presents the empirical attackability curves. In both charts
$\widehat A_{n,w}(\eta)$ is sigmoidal, rising from near zero to
$\approx 0.76$ (soft-token) and $\approx 0.97$ (cloud),
with DKW half-widths $\pm 0.099$ and $\pm 0.100$. Both empirical-inclusion
overlays lie above $\widehat A_{n,w}$, and the $\Sigma_w$ overlay below the
$\lambda^*$ overlay at matching $\beta$, as required; but both saturate at one
for small budgets, $\lambda^*$ almost immediately, $\Sigma_w$ by $\eta\approx 0.2$,
so over most of the range they are valid but close to vacuous. The wide gap
between $\widehat A_{n,w}$ and either overlay is the curve-level counterpart of
the large slack in Figure~\ref{fig:sigma-lambda}. The soft-token and cloud curves agree in shape and
lie within overlapping DKW bands.

\begin{figure}[htbp]
    \centering

    \begin{subfigure}{0.48\textwidth}
        \centering
        \includegraphics[width=\textwidth]{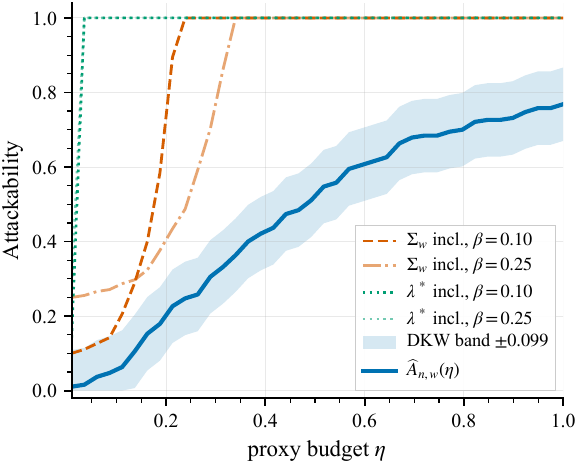}
        \caption{Soft-token relaxation}
        \label{fig:attackability_curve}
    \end{subfigure}
    \hfill
    \begin{subfigure}{0.48\textwidth}
        \centering
        \includegraphics[width=\textwidth]{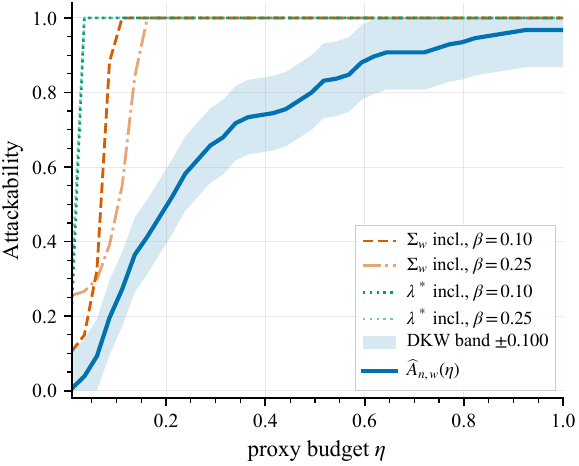}  
        \caption{Paraphrase cloud}
        \label{fig:attackability_curve_cloud}
    \end{subfigure}

\caption{Empirical attackability curves $\widehat A_{n,w}(\eta)$ against the
proxy budget $\eta$, with DKW bands and the empirical-inclusion overlays of
Theorem~\ref{thm:Aeta-margin} and Remark~\ref{rem:Aeta-sigma}
($\beta\in\{0.10,0.25\}$). (a)~soft-token chart; (b)~paraphrase cloud.}
    \label{fig:attackability_curves}
\end{figure}

Finite search, Figure~\ref{fig:finite-search}, lies below the cloud-chart $\widehat A_{n,w}$ at every
budget and every $K$ and increases monotonically in $K$: the per-input ceiling
rises from $\approx 0.11$ at $K=4$ through $\approx 0.31$ at $K=25$ to
$\approx 0.39$ for the best-of-all oracle, against a linearised
$\widehat A_{n,w}\approx 0.97$. No candidate is proxy-valid below
$\eta\approx 0.25$, so all curves are zero there. The coverage panel shows the
median normalised gap along $r_w$ holding at one, fully uncovered, until
$\eta\approx 0.35$ and settling near $0.83$ thereafter: the candidate sets
barely populate the readout-relevant axis, so the finite-search shortfall is the
expected consequence of sampling geometry, the quantitative form of the
failure-to-certify asymmetry of Section~\ref{sec:Finite Paraphrase}. The gap between the oracle and
$\widehat A_{n,w}$ reflects both this coverage shortfall and the first-order
over-prediction quantified next, so we read the asymmetry from the coverage
diagnostic rather than from that gap alone.

\begin{figure}[hbtp]
\centering
\includegraphics[width=\textwidth]{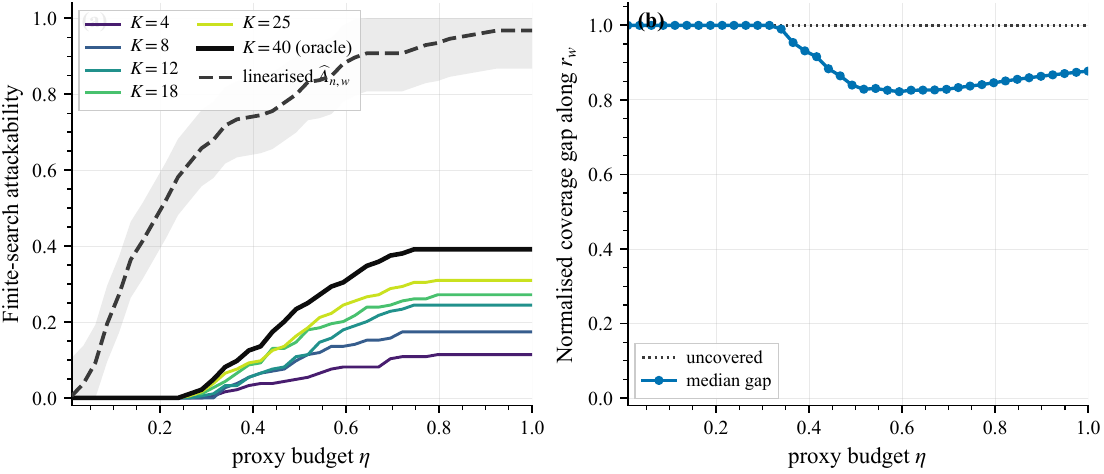}
\caption{Finite-search attackability in the cloud chart against the proxy budget
$\eta$. (a)~$A^{\mathrm{fin}}_w(\eta;\mathcal P_K)$ for $K\in\{4,8,12,18,25\}$ and
the best-of-all-candidates oracle, against the linearised $\widehat A_{n,w}$;
(b)~median normalised coverage gap along the readout-aligned direction $r_w(x)$.}
\label{fig:finite-search}
\end{figure}

Finally, the predictive-validity test, Figure~\ref{fig:predictive-validity}, confirms that $Z_w$ orders
attackability as intended: the Spearman correlation between $Z_w(x)$ and the
smallest realised flip budget is $\rho=0.58$ ($n=184$), and the AUC of $-Z_w$
for predicting a realised flip within the median budget is $0.91$. Every input
lies above the identity, so the linearised $Z_w$ systematically under-states the
realised flip budget (the saturation seen in Figure~\ref{fig:linearisation-error}) and $Z_w$ should be
read as a ranking index of attackability rather than as a pointwise predictor of
the realised budget.

\begin{figure}[hbtp]
\centering
\includegraphics[width=\textwidth]{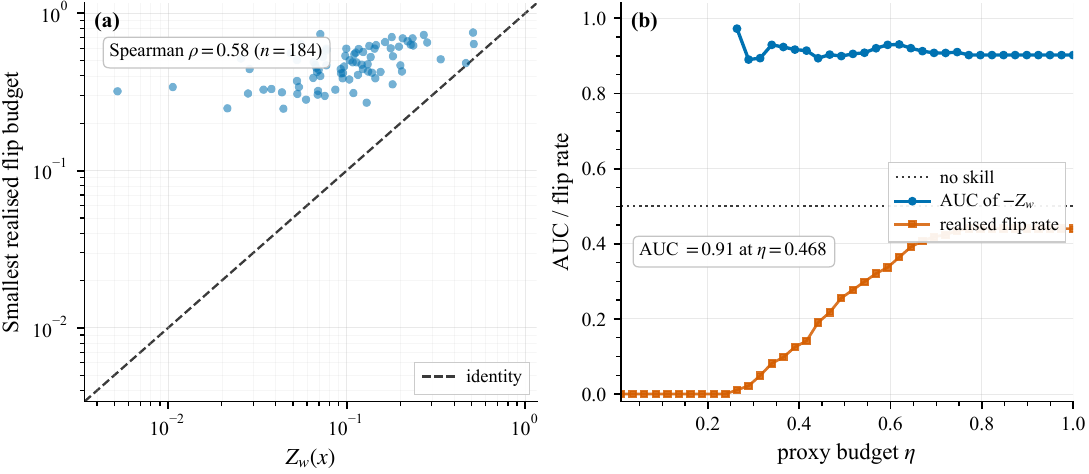}
\caption{Predictive validity of $Z_w$ in the cloud chart. (a)~smallest realised
flip budget versus $Z_w(x)$ with the identity line; (b)~AUC of $-Z_w$ and the
realised flip rate against the proxy budget $\eta$.}
\label{fig:predictive-validity}
\end{figure}

\FloatBarrier

\section{Discussion}\label{sec:Discussion}

The main contribution of the paper is the identification of the local geometric
object that controls semantic attackability in a two-embedding setting. The
matrix pencil $(A(x),B(x))$ describes how target and proxy embeddings assign
length to the same local paraphrase directions, while the target-space matrix
$S(x)=J_M(x)B(x)^{-1}J_M(x)^\top$ is the object seen by an affine readout. The
former gives the worst-direction displacement $\lambda^*(x)$; the latter gives
the readout-specific displacement $\Sigma_w(x)=w^\top S(x)w$ and hence the
adjusted margin.

The statistical bounds in the paper serve different purposes. The DKW
bound in Section~\ref{sec:population} is sharp and simple, but applies to a
single fixed readout. The VC bound in Section~\ref{sec:hypotheses} is fully distribution-free and
uniform over data-dependent affine readouts. The elementary lifting in
Proposition~\ref{prop:vc-affine} gives an order-$d^2$ bound by representing
the attackability event as a quadratic feature lift, and the
polynomial-threshold argument in Proposition~\ref{prop:vc-improved}
sharpens this to order $d$.
 The margin bounds in Section~\ref{sec:margins} are more
scale-sensitive: they depend on the empirical distribution of
$y f_\theta(x)-\eta\sqrt{w^\top S(x)w}$ and on norm, tail, or trace control of the local
sensitivity. The covering-number version is explicit but dimension-dependent,
whereas the Rademacher version replaces the ambient dimension by the empirical
trace sensitivity $n^{-1}\sum_i\operatorname{tr}S(X_i)$. This mirrors the usual distinction between hard-indicator VC bounds
and margin-based generalisation bounds.

Empirically, the experiments in Section~\ref{sec:experiments} confirm the paper's central claims on a deployed classifier. The pointwise inequality $\Sigma_w(x)\le\lambda^*(x)$ of
Theorem~\ref{thm:sigma-lambda} holds without exception in both charts, and the attackability margin $Z_w(x)$ is empirically predictive of vulnerability: it separates inputs that admit a proxy-valid flip from those that do not with AUC $\approx 0.91$, and ranks inputs by realised flip budget at Spearman's $\rho\approx 0.58$, i.e., direct evidence that the local geometry identifies where the deployed model is fragile, and support for reading $Z_w$ as a screening index rather than, given the saturation of Figure~\ref{fig:linearisation-error}, a pointwise predictor of the exact budget. The remaining findings sharpen this picture rather than qualify it. The one-and-a-half to two orders of magnitude of slack between $\lambda^*(x)$ and $\Sigma_w(x)$ shows that FinBERT's readout sits far from the worst local paraphrase direction; this is precisely why the readout-free certificate of Corollary~\ref{cor:worstcase-flip} and the margin-tail bound of Theorem~\ref{thm:Aeta-margin} are conservative on FinBERT, and why the adjusted margin is built from the readout-specific
$S(x)$, whose quantile bound (Remark~\ref{rem:Aeta-sigma}) is correspondingly tighter. Likewise, the gap between finite search and the linearised attackability curve is the failure-to-certify asymmetry of Section~\ref{sec:Finite Paraphrase} made quantitative: the candidate clouds barely populate the readout-relevant direction (oracle $\approx 0.39$ against a linearised $\approx 0.97$), so a failed search certifies robustness only against its own candidate set. The conservative quantities are conservative for reasons the theory names in advance, while the quantity the theory proposes for diagnosis, $Z_w$, predicts realised attacks.

Several limitations remain. First, the theory is local and first-order. The
nonlinear error terms in Theorem~\ref{thm:linearization-error} must be small for
the linearised certificate to be predictive at a finite budget. Second, the local
charts and metric matrices are assumed as part of the model; the paper does not
prescribe a particular procedure for estimating them from a discrete paraphrase
space. Third, finite paraphrase search gives explicit attacks when it succeeds,
but failure of a finite search procedure is only a certificate relative to the
generated candidate set unless a coverage condition such as
Proposition~\ref{prop:finite-search-cover} is available.

Several natural extensions are not pursued here. 
One could study finite unions of attackability events to model compound
fragility detectors that fire when any one of several readouts is locally
attackable; the VC dimension of an $m$-fold union of a class with VC
dimension $d$ is at most $O(dm\log m)$ by standard results on Boolean
combinations, and the corresponding
generalisation bounds follow by substitution into Theorem~\ref{thm:vc-uniform}. One could also formulate
asymmetric, Neyman--Pearson-style certificates that separately control false
positive and false negative declarations of fragility. Finally, nonlinear
readout heads can be handled locally by replacing the affine direction with the
gradient of the relevant score or logit gap, but uniform bounds for such heads
would also require complexity control of the associated gradient class. These
extensions would enlarge the scope of the present paper, whose focus is the
geometric object $S(x)$ and its consequences for affine readouts.

The binary affine-readout assumption is mainly expositional. For multiclass
linear logits, the relevant local quantity at a labelled example is the
smallest gap between the true-class logit and any competing-class logit,
adjusted by the worst-case proxy-bounded displacement along the
corresponding gap direction. The local sensitivity matrix $S(x)$ enters
in the same way as in the binary case, contracted against each gap
direction rather than against a single readout direction, and the
uniform bounds extend by the usual multiclass adjustments to the margin
class.

\section*{Acknowledgments}

This work is supported by The London School of Economics and Political Science, and the Economic and Social Research Council (ESRC) under the “Diversity and Productivity: from Education to Work” (DAPEW) project [Grant Ref: ES/W010224/1].

\bibliographystyle{apalike}
\bibliography{References_martin}

@inproceedings{zhang2022manifold,
  title={A manifold view of adversarial risk},
  author={Zhang, Wenjia and Zhang, Yikai and Hu, Xiaoling and Goswami, Mayank and Chen, Chao and Metaxas, Dimitris N},
  booktitle={International Conference on Artificial Intelligence and Statistics},
  pages={11598--11614},
  year={2022},
  organization={PMLR}
}

@inproceedings{carlini2017towards,
  title={Towards evaluating the robustness of neural networks},
  author={Carlini, Nicholas and Wagner, David},
  booktitle={2017 IEEE symposium on security and privacy (SP)},
  pages={39--57},
  year={2017},
  organization={IEEE}
}

@article{malo2014good,
  title={Good debt or bad debt: Detecting semantic orientations in economic texts},
  author={Malo, Pekka and Sinha, Ankur and Korhonen, Pekka and Wallenius, Jyrki and Takala, Pyry},
  journal={Journal of the Association for Information Science and Technology},
  volume={65},
  number={4},
  pages={782--796},
  year={2014},
  publisher={Wiley Online Library}
}

@inproceedings{reimers2019sentence,
  title={Sentence-bert: Sentence embeddings using siamese bert-networks},
  author={Reimers, Nils and Gurevych, Iryna},
  booktitle={Proceedings of the 2019 conference on empirical methods in natural language processing and the 9th international joint conference on natural language processing (EMNLP-IJCNLP)},
  pages={3982--3992},
  year={2019}
}

@article{araci2019finbert,
  title={Finbert: Financial sentiment analysis with pre-trained language models},
  author={Araci, Dogu},
  journal={arXiv preprint arXiv:1908.10063},
  year={2019}
}

@article{baydin2018automatic,
  title={Automatic differentiation in machine learning: a survey},
  author={Baydin, Atilim Gunes and Pearlmutter, Barak A and Radul, Alexey Andreyevich and Siskind, Jeffrey Mark},
  journal={Journal of Machine Learning Research},
  volume={18},
  number={153},
  pages={1--43},
  year={2018}
}

@inproceedings{szegedy2014intriguing,
  author    = {Szegedy, Christian and Zaremba, Wojciech and Sutskever, Ilya and Bruna, Joan and Erhan, Dumitru and Goodfellow, Ian and Fergus, Rob},
  title     = {Intriguing Properties of Neural Networks},
  booktitle = {International Conference on Learning Representations},
  year      = {2014},
  note      = {arXiv:1312.6199}
}

@inproceedings{goodfellow2015explaining,
  author    = {Goodfellow, Ian J. and Shlens, Jonathon and Szegedy, Christian},
  title     = {Explaining and Harnessing Adversarial Examples},
  booktitle = {International Conference on Learning Representations},
  year      = {2015},
  note      = {arXiv:1412.6572}
}

@inproceedings{madry2017towards,
  author    = {Madry, Aleksander and Makelov, Aleksandar and Schmidt, Ludwig and Tsipras, Dimitris and Vladu, Adrian},
  title     = {Towards Deep Learning Models Resistant to Adversarial Attacks},
  booktitle = {International Conference on Learning Representations},
  year      = {2018},
  note      = {arXiv:1706.06083}
}

@inproceedings{tsipras2018robustness,
  author    = {Tsipras, Dimitris and Santurkar, Shibani and Engstrom, Logan and Turner, Alexander and Madry, Aleksander},
  title     = {Robustness May Be at Odds with Accuracy},
  booktitle = {International Conference on Learning Representations},
  year      = {2019},
  note      = {arXiv:1805.12152}
}

@inproceedings{alzantot2018generating,
  author    = {Alzantot, Moustafa and Sharma, Yash and Elgohary, Ahmed and Ho, Bo-Jhang and Srivastava, Mani and Chang, Kai-Wei},
  title     = {Generating Natural Language Adversarial Examples},
  booktitle = {Proceedings of the 2018 Conference on Empirical Methods in Natural Language Processing},
  pages     = {2890--2896},
  year      = {2018},
  address   = {Brussels, Belgium},
  publisher = {Association for Computational Linguistics},
  doi       = {10.18653/v1/D18-1316}
}

@inproceedings{jin2020bert,
  author    = {Jin, Di and Jin, Zhijing and Zhou, Joey Tianyi and Szolovits, Peter},
  title     = {Is {BERT} Really Robust? A Strong Baseline for Natural Language Attack on Text Classification and Entailment},
  booktitle = {Proceedings of the AAAI Conference on Artificial Intelligence},
  volume    = {34},
  pages     = {8018--8025},
  year      = {2020},
  doi       = {10.1609/aaai.v34i05.6311}
}

@article{TURETKEN2026107698,
  author  = {Can T{\"u}retken, Aysun and Leippold, Markus},
  title   = {Battle of Transformers: Adversarial Attacks on Financial Sentiment Models},
  journal = {Journal of Banking \& Finance},
  volume  = {188},
  pages   = {107698},
  year    = {2026},
  doi     = {10.1016/j.jbankfin.2026.107698}
}

@inproceedings{guo2021gradient,
  author    = {Guo, Chuan and Sablayrolles, Alexandre and J{\'e}gou, Herv{\'e} and Kiela, Douwe},
  title     = {Gradient-based Adversarial Attacks against Text Transformers},
  booktitle = {Proceedings of the 2021 Conference on Empirical Methods in Natural Language Processing},
  pages     = {5747--5757},
  year      = {2021},
  address   = {Online and Punta Cana, Dominican Republic},
  publisher = {Association for Computational Linguistics},
  doi       = {10.18653/v1/2021.emnlp-main.464}
}

@inproceedings{jang2016categorical,
  author    = {Jang, Eric and Gu, Shixiang and Poole, Ben},
  title     = {Categorical Reparameterization with {Gumbel-Softmax}},
  booktitle = {International Conference on Learning Representations},
  year      = {2017},
  note      = {arXiv:1611.01144}
}

@article{papernot2016transferability,
  author  = {Papernot, Nicolas and McDaniel, Patrick and Goodfellow, Ian},
  title   = {Transferability in Machine Learning: From Phenomena to Black-Box Attacks Using Adversarial Samples},
  journal = {arXiv preprint arXiv:1605.07277},
  year    = {2016}
}

@inproceedings{demontis2019adversarial,
  author    = {Demontis, Ambra and Melis, Marco and Pintor, Maura and Jagielski, Matthew and Biggio, Battista and Oprea, Alina and Nita-Rotaru, Cristina and Roli, Fabio},
  title     = {Why Do Adversarial Attacks Transfer? Explaining Transferability of Evasion and Poisoning Attacks},
  booktitle = {28th USENIX Security Symposium (USENIX Security 19)},
  pages     = {321--338},
  year      = {2019}
}

@inproceedings{tramer2017ensemble,
  author    = {Tram{\`e}r, Florian and Kurakin, Alexey and Papernot, Nicolas and Goodfellow, Ian and Boneh, Dan and McDaniel, Patrick},
  title     = {Ensemble Adversarial Training: Attacks and Defenses},
  booktitle = {International Conference on Learning Representations},
  year      = {2018},
  note      = {arXiv:1705.07204}
}

@article{vapnik1971uniform,
  author  = {Vapnik, Vladimir N. and Chervonenkis, Alexey Ya.},
  title   = {On the Uniform Convergence of Relative Frequencies of Events to Their Probabilities},
  journal = {Theory of Probability and Its Applications},
  volume  = {16},
  number  = {2},
  pages   = {264--280},
  year    = {1971},
  doi     = {10.1137/1116025}
}

@book{anthony2009neural,
  author    = {Anthony, Martin and Bartlett, Peter L.},
  title     = {Neural Network Learning: Theoretical Foundations},
  publisher = {Cambridge University Press},
  year      = {1999},
  note      = {Paperback edition 2009}
}

@article{anthony1995classification,
  author  = {Anthony, Martin},
  title   = {Classification by Polynomial Surfaces},
  journal = {Discrete Applied Mathematics},
  volume  = {61},
  number  = {2},
  pages   = {91--103},
  year    = {1995},
  doi     = {10.1016/0166-218X(94)00008-2}
}

@book{shalev2014understanding,
  author    = {Shalev-Shwartz, Shai and Ben-David, Shai},
  title     = {Understanding Machine Learning: From Theory to Algorithms},
  publisher = {Cambridge University Press},
  year      = {2014}
}

@article{bartlett2002rademacher,
  author  = {Bartlett, Peter L. and Mendelson, Shahar},
  title   = {Rademacher and {Gaussian} Complexities: Risk Bounds and Structural Results},
  journal = {Journal of Machine Learning Research},
  volume  = {3},
  pages   = {463--482},
  year    = {2002}
}

@inproceedings{maurer2016vector,
  author    = {Maurer, Andreas},
  title     = {A vector-contraction inequality for {R}ademacher complexities},
  booktitle = {Algorithmic Learning Theory (ALT 2016)},
  publisher = {Springer},
  pages     = {3--17},
  year      = {2016}
}

@article{goldberg1995bounding,
  author  = {Goldberg, Paul W. and Jerrum, Mark R.},
  title   = {Bounding the {V}apnik-{C}hervonenkis dimension of concept classes parameterized by real numbers},
  journal = {Machine Learning},
  volume  = {18},
  number  = {2-3},
  pages   = {131--148},
  year    = {1995}
}

\appendix

\newpage
\clearpage

\section{Proofs}

\subsection{Proof of Lemma~\ref{lem:localrep}}
  We prove the claim for a generic map $E:U\to S^{d-1}$ with Jacobian
$J$ at the origin. By differentiability,
\[
 E(u)-E(0)=Ju+r(u),
 \qquad
 \|r(u)\|=o(\|u\|).
\]
  Hence
\[
 \|E(u)-E(0)\|_2^2
 =\|Ju+r(u)\|_2^2
 =u^\top J^\top Ju+2(Ju)^\top r(u)+\|r(u)\|_2^2
 =u^\top J^\top Ju+o(\|u\|^2).
\]
  Applying this argument to $E_P$ and $E_M$ gives
\eqref{eq:quad-P} and \eqref{eq:quad-M}. Positive semidefiniteness follows
from $v^\top J^\top Jv=\|Jv\|^2\geq0$.

\subsection{Proof of Proposition~\ref{prop:solution}}

If $A=0$, then $u^\top A u = 0$ for every feasible $u$, and the optimal value of \eqref{eq:local-problem} is $0=\eta^2\lambda_{\max}(B^{-1}A)$. Assume henceforth that $A\neq 0$.

The feasible set $\mathcal{F}:=\{u\in\mathbb{R}^q : u^\top B u\leq\eta^2\}$ is closed and, because $B\succ 0$, bounded; hence compact. Continuity of $u\mapsto u^\top A u$ then guarantees a maximiser $u^*\in\mathcal{F}$. Since $A\succeq 0$ and $A\neq 0$, some $v$ has $v^\top A v>0$, and rescaling $v$ to satisfy $v^\top B v=\eta^2$ produces a feasible point with strictly positive objective; hence $u^{*\top}A u^*>0$, and in particular $u^*\neq 0$. If $u^{*\top}B u^*<\eta^2$, choose $\alpha>1$ with $\alpha^2 u^{*\top}B u^*=\eta^2$; then $\alpha u^*$ is feasible with $(\alpha u^*)^\top A(\alpha u^*)=\alpha^2 u^{*\top}A u^*>u^{*\top}A u^*$, contradicting optimality. Therefore $u^{*\top}B u^*=\eta^2$, and $u^*$ also maximises the equality-constrained problem
\begin{equation}\label{eq:objective_proof}
\max_{u\in\mathbb{R}^q} u^\top A u \quad\text{subject to}\quad u^\top B u=\eta^2.
\end{equation}

The constraint function $g(u):=u^\top B u-\eta^2$ has gradient $\nabla g(u)=2Bu$, which is nonzero whenever $u\neq 0$, so the constraint is regular at $u^*$. The Lagrange multiplier theorem applied to the Lagrangian
\begin{equation}
\mathcal{L}(u,\mu):=u^\top A u-\mu(u^\top B u-\eta^2)
\end{equation}
gives $\mu\in\mathbb{R}$ such that $\nabla_u\mathcal{L}(u^*,\mu)=0$, that is,
\begin{equation}\label{eq:gev_problem}
A u^*=\mu B u^*,
\end{equation}
the generalised eigenvalue equation for the pencil $(A,B)$. Left-multiplying \eqref{eq:gev_problem} by $u^{*\top}$ and using the active constraint gives $u^{*\top}A u^*=\mu\eta^2$. Hence the optimal value equals $\eta^2\mu$ for some generalised eigenvalue $\mu$ of $(A,B)$, and
\begin{equation}
\sup_{u\in\mathcal{F}} u^\top A u \leq \eta^2\mu_{\max},
\end{equation}
where $\mu_{\max}$ is the largest generalised eigenvalue. The reverse inequality follows by exhibiting a feasible point that attains it: let $u_{\max}\neq 0$ satisfy $A u_{\max}=\mu_{\max}B u_{\max}$, and rescale so that $u_{\max}^\top B u_{\max}=\eta^2$. Then $u_{\max}^\top A u_{\max}=\mu_{\max}\eta^2$, attaining the bound.

Since $B$ is invertible, $Au=\mu Bu$ is equivalent to $B^{-1}A u=\mu u$, so $\mu_{\max}=\lambda_{\max}(B^{-1}A)$. The Cholesky decomposition $B=LL^\top$ with $L$ invertible, combined with the substitution $z=L^\top u$, identifies the generalised eigenvalues of $(A,B)$ with the eigenvalues of the symmetric positive semidefinite matrix $L^{-1}A L^{-\top}$, so they are real and non-negative. The optimal value of \eqref{eq:local-problem} is therefore $\eta^2\lambda_{\max}(B^{-1}A)$, attained at any top generalised eigenvector of $(A,B)$ scaled to satisfy $u^{*\top}B u^*=\eta^2$.

%\subsection{Proof of Proposition~\ref{prop:solution}}
%If $A=0$, the value of \eqref{eq:local-problem} is zero and every
%feasible point is optimal. Otherwise set $z=B^{1/2}u$, so that
%$u=B^{-1/2}z$. The constraint becomes $\|z\|_2\leq\eta$, and the objective
%is
%\[
% u^\top Au=z^\top B^{-1/2}AB^{-1/2}z.
%\]
%  Let $C=B^{-1/2}AB^{-1/2}$. Since $C$ is symmetric positive semidefinite,
%the Rayleigh--Ritz theorem gives
%\[
% \max_{\|z\|_2\leq\eta} z^\top Cz
% =\eta^2\lambda_{\max}(C)
% =\eta^2\lambda_{\max}(A,B).
%\]
% The maximisers are $z^*=\eta v$, where $v$ is a unit top eigenvector of
%$C$. Transforming back, $u^*=B^{-1/2}z^*$ satisfies
%$Au^*=\lambda_{\max}(A,B)Bu^*$  and   $u^{*\top}Bu^*=\eta^2$.

  \subsection{Proof of Proposition~\ref{prop:chart-inv}}
By the chain rule, $\tilde J_M=J_MT$  and   $\tilde J_P=J_PT$, which gives
\eqref{eq:congruence}. If  $B$ is invertible,   then
\[
(T^\top BT)^{-1}=T^{-1}B^{-1}T^{-\top},
\]
and therefore
\[
 \tilde B^{-1}\tilde A
 =(T^\top BT)^{-1}(T^\top AT)
 =T^{-1}B^{-1}T^{-\top}T^\top AT
 =T^{-1}B^{-1}AT,
\]
  so \eqref{eq:similarity} holds and $\tilde B^{-1}\tilde A$ is similar to
$B^{-1}A$. Equivalently,
\[
 \det(\tilde A-\lambda\tilde B)=\det(T)^2\det(A-\lambda B),
\]
so the generalised eigenvalues are preserved. If
$Au^*=\lambda^*Bu^*$ and $\tilde u^*=T^{-1}u^*$, then
\[
 \tilde A\tilde u^*=T^\top Au^*
 =\lambda^*T^\top Bu^*
 =\lambda^*\tilde B\tilde u^*.
\]
Finally, $\tilde J_M\tilde u^*=(J_MT)(T^{-1}u^*)=J_Mu^*$.

\subsection{Proof of Theorem~\ref{thm:flip}}
Fix $x$ and use the shorthand $J_M=J_M(x)$ and $B=B(x)$ introduced in the
statement of the theorem. Let $a=J_M^\top w$. The quantity to be maximised is
$|a^\top u|$ subject
to $u^\top Bu\leq\eta^2$. Under the $B$-inner  product
$\langle u,v\rangle_B=u^\top Bv$,
\[
 a^\top u
 =\langle B^{-1}a,u\rangle_B
 \leq
 \sqrt{a^\top B^{-1}a}\sqrt{u^\top Bu}
 \leq
 \eta\sqrt{a^\top B^{-1}a}.
\]
  Since $a^\top B^{-1}a=w^\top J_MB^{-1}J_M^\top w=\Sigma_w(x)$, the
supremum of $|w^\top J_Mu|$  is at most   $\eta\sqrt{\Sigma_w(x)}$. If
$\Sigma_w(x)>0$, equality is attained at
\[
 u=\pm\eta\frac{B^{-1}J_M^\top w}{\sqrt{\Sigma_w(x)}}.
\]
Choosing the sign $-\operatorname{sign}(s_0)$ gives the direction toward
the decision boundary. If $\Sigma_w(x)=0$, then $a=0$ and the supremum is
zero. The linearised score crosses the boundary precisely when the
maximum achievable movement toward the boundary exceeds
$\gamma_w(x)=|s_0|$, which gives \eqref{eq:flip-condition}.

\subsection{Proof of Theorem~\ref{thm:linearization-error}}
\label{app:p:proof-r-eta}

Let
\[
R_M(u):=E_M(u)-E_M(0)-J_Mu .
\]
For fixed $u$, define the curve
\[
\phi(t):=E_M(tu), \qquad t\in[0,1].
\]
Assume that the line segment $\{tu:t\in[0,1]\}$ is contained in the local chart $U$. Since
\[
J_M=\left.\frac{\partial E_M}{\partial u}\right|_{u=0},
\]
the chain rule gives
\[
\phi'(t)=\frac{\partial E_M}{\partial u}(tu)u
\]
and
\[
\phi''(t)
=
\frac{\partial^2 E_M}{\partial u^2}(tu)[u,u],
\]
where
\[
\frac{\partial^2 E_M}{\partial u^2}(v)
:
\mathbb R^q\times \mathbb R^q
\to
\mathbb R^{d_M}
\]
is the second derivative of $E_M$ with respect to the local coordinate $u$, interpreted as
a bilinear map. Using the fundamental theorem of calculus twice, applied componentwise to the
vector-valued curve $\phi$, we have
\[
\begin{aligned}
\phi(1)-\phi(0)-\phi'(0)
&=
\int_0^1\{\phi'(s)-\phi'(0)\}\,ds \\
&=
\int_0^1\int_0^s \phi''(t)\,dt\,ds \\
&=
\int_0^1(1-t)\phi''(t)\,dt .
\end{aligned}
\]
Therefore,
\[
R_M(u)
=
\int_0^1(1-t)
\frac{\partial^2 E_M}{\partial u^2}(tu)[u,u]\,dt .
\]

For $\eta>0$, define
\[
U_\eta
:=
\{tu:t\in[0,1],\ u^\top Bu\le \eta^2\}.
\]
For $\eta$ sufficiently small, $U_\eta\subset U$. Define
\[
C_M
:=
\sup_{v\in U_\eta}
\left\|
\frac{\partial^2 E_M}{\partial u^2}(v)
\right\|_{\mathrm{bil}},
\]
where
\[
\left\|
\frac{\partial^2 E_M}{\partial u^2}(v)
\right\|_{\mathrm{bil}}
:=
\sup_{\|a\|_2\le 1,\ \|b\|_2\le 1}
\left\|
\frac{\partial^2 E_M}{\partial u^2}(v)[a,b]
\right\|_2 .
\]
Since $E_M$ is $C^2$, this quantity is finite for sufficiently small $\eta$. Hence
\[
\begin{aligned}
\|R_M(u)\|_2
&\le
\int_0^1(1-t)
\left\|
\frac{\partial^2 E_M}{\partial u^2}(tu)[u,u]
\right\|_2\,dt \\
&\le
\int_0^1(1-t)C_M\|u\|_2^2\,dt \\
&=
\frac{C_M}{2}\|u\|_2^2 .
\end{aligned}
\]
Therefore, using Cauchy--Schwarz and $\|w\|_2=1$,
\[
\begin{aligned}
r_M(\eta)
&=
\sup_{u^\top Bu\le \eta^2}
\left|
w^\top\{E_M(u)-E_M(0)-J_Mu\}
\right| \\
&=
\sup_{u^\top Bu\le \eta^2}
|w^\top R_M(u)| \\
&\le
\sup_{u^\top Bu\le \eta^2}
\|w\|_2\|R_M(u)\|_2 \\
&\le
\frac{C_M}{2}
\sup_{u^\top Bu\le \eta^2}
\|u\|_2^2 .
\end{aligned}
\]
Since $B\succ0$, let
\[
\beta:=\lambda_{\min}(B)>0 .
\]
Then $B\succeq \beta I$, so
\[
u^\top Bu\ge \beta\|u\|_2^2 .
\]
Thus, on the feasible set $u^\top Bu\le \eta^2$,
\[
\|u\|_2^2\le \frac{\eta^2}{\beta}.
\]
Substituting gives
\[
r_M(\eta)
\le
\frac{C_M\eta^2}{2\beta}
=
O(\eta^2/\beta),
\qquad
\eta\downarrow0 .
\]

The same argument applied to $E_P$ gives, with
\[
R_P(u):=E_P(u)-E_P(0)-J_Pu,
\]
the bound
\[
\|R_P(u)\|_2
\le
\frac{C_P}{2}\|u\|_2^2,
\]
where
\[
C_P
:=
\sup_{v\in U_\eta}
\left\|
\frac{\partial^2 E_P}{\partial u^2}(v)
\right\|_{\mathrm{bil}} .
\]
Consequently,
\[
\left|
\|E_P(u)-E_P(0)\|_2-\|J_Pu\|_2
\right|
\le
\|R_P(u)\|_2
\le
\frac{C_P}{2}\|u\|_2^2 .
\]
Since
\[
\|J_Pu\|_2^2
=
u^\top J_P^\top J_Pu
=
u^\top Bu,
\]
replacing the exact proxy displacement
$\|E_P(u)-E_P(0)\|_2$ by its linearised version
$(u^\top Bu)^{1/2}$ incurs a second-order error in $\|u\|_2$.
This proves the claimed finite-radius error bound.

\subsection{Proof of Theorem~\ref{thm:sigma-lambda}}
  Set $C=J_MB^{-1/2}$. Then
\[
 \Sigma_w(x)=w^\top CC^\top w.
\]
For unit $w$,  the Rayleigh quotient bound   gives
\[
 \Sigma_w(x)\leq \lambda_{\max}(CC^\top).
\]
  The nonzero eigenvalues of $CC^\top$ and $C^\top C$ coincide, and
\[
 C^\top C=B^{-1/2}J_M^\top J_MB^{-1/2}=B^{-1/2}AB^{-1/2}.
\]
  Therefore
\[
 \lambda_{\max}(CC^\top)
 =\lambda_{\max}\left(B^{-1/2}AB^{-1/2}\right)
 =\lambda^*(x),
\]
which proves \eqref{eq:sigma-leq-lambda}. Equality holds when $w$ is a
top eigenvector of $CC^\top=J_MB^{-1}J_M^\top$.

\subsection{Proof of Theorem~\ref{thm:Aeta-margin}}
 By Theorem~\ref{thm:sigma-lambda},   $\Sigma_w(x)\leq\lambda^*(x)$ whenever
Theorem~\ref{thm:flip} applies. If $Z_w(x)<\eta$ and
$\lambda^*(x)\leq\Lambda$, then
\[
 \gamma_w(x)<\eta\sqrt{\Sigma_w(x)}
 \leq \eta\sqrt{\lambda^*(x)}
 \leq \eta\sqrt{\Lambda}.
\]
Thus
\[
 \{Z_w(x)<\eta\}\subseteq
 \{\gamma_w(x)<\eta\sqrt{\Lambda}\}\cup\{\lambda^*(x)>\Lambda\}.
\]
Taking probabilities gives \eqref{eq:Aeta-bound-tail}. The almost-sure
bounded case and the quantile form are immediate specialisations.

\subsection{Proof of Theorem~\ref{thm:Aeta-empirical}}
  For fixed $(w,b)$, the random variables $Z_w(x_1),\ldots,Z_w(x_n)$ are
 i.i.d.  \ with left-limit distribution function
$F_{Z_w}^{-}(\eta)=\mathcal{A}_w(\eta)$. The empirical curve
$\hat{\mathcal{A}}_{n,w}$ is the corresponding empirical distribution
function with strict threshold. The Dvoretzky--Kiefer--Wolfowitz
inequality with Massart's sharp constant gives, for every $\epsilon>0$,
\[
 \mathbb{P}\left(
 \sup_{\eta>0}\left|\hat{\mathcal{A}}_{n,w}(\eta)-\mathcal{A}_w(\eta)\right|>
 \epsilon\right)
 \leq 2\exp(-2n\epsilon^2).
\]
  Setting the right-hand side equal to $\delta$ and  solving for $\epsilon$
  gives \eqref{eq:Aeta-dkw}.

\subsection{Proof of Theorem~\ref{thm:vc-uniform}}

Write $P$ for expectation under the population distribution of $X$ and
$P_n$ for the empirical measure. For each pair $(\theta,\eta)$, the function
$G_{\theta,\eta}$ is binary-valued, and
\[
P_nG_{\theta,\eta}=\widehat{\mathcal A}_{n,\theta}(\eta),
\qquad
PG_{\theta,\eta}=\mathcal A_\theta(\eta).
\]
Therefore the left-hand side of \eqref{eq:vc-uniform-bound} is exactly
\[
\sup_{g\in\mathcal G_\Theta}|(P_n-P)g|.
\]
The standard VC uniform convergence inequality for a binary class of VC
dimension $v$ gives, with probability at least $1-\delta$,
\[
\sup_{g\in\mathcal G_\Theta}|(P_n-P)g|
\leq
C\sqrt{\frac{v\log(en)+\log(1/\delta)}{n}},
\]
for a universal constant $C$; see, for example,
\citet{vapnik1971uniform} or \citet{anthony2009neural}. Substituting the
identities above gives the result. The point of the theorem is that the
supremum is taken over both the readout $\theta$ and the budget $\eta$, so the
bound remains valid even if the readout is selected in a data-dependent way.

\subsection{Proof of Proposition~\ref{prop:vc-affine}}

Fix $(\theta,\eta)$, with $\theta=(w,b)$. Let
\[
\tilde z(x):=(z(x),1)\in\mathbb R^{d_M+1},
\qquad
\tilde w:=(w,b)\in\mathbb R^{d_M+1}.
\]
The attackability event is
\[
|w^\top z(x)+b|<\eta\sqrt{w^\top S(x)w}.
\]

Since $S(x)\succeq0$, the expression under the square root is nonnegative.
Both sides of the strict inequality are therefore nonnegative, and squaring
is an equivalent transformation; the case $w^\top S(x)w=0$ is handled the
same way, since then both the original and the squared inequalities fail
trivially. Hence the event is equivalently
\begin{equation}\label{eq:proof-vc-quadratic}
(\tilde w^\top\tilde z(x))^2-\eta^2w^\top S(x)w<0.
\end{equation}

We now write the quadratic comparison as a linear threshold after an explicit
feature lift. This is the same elementary lifting used for classification by
polynomial surfaces, specialized here to degree two and augmented with the
entries of $S(x)$ \citep{anthony1995classification}. Write $z_i(x)$ for the
coordinates of $z(x)$ and $S_{ij}(x)$ for the entries of $S(x)$. Define
\begin{equation}\label{eq:explicit-lift-Psi}
\begin{aligned}
\Psi(x):=\bigl(&1,\ z_1(x),\ldots,z_{d_M}(x),\\
&\bigl(z_i(x)z_j(x)\bigr)_{1\leq i\leq j\leq d_M},\quad
\bigl(S_{ij}(x)\bigr)_{1\leq i\leq j\leq d_M}\bigr)\in\mathbb R^m.
\end{aligned}
\end{equation}
This feature vector consists of a constant coordinate, the ordinary embedding
coordinates, their distinct quadratic products, and the distinct entries of
$S(x)$. Expanding the two terms in \eqref{eq:proof-vc-quadratic} gives
\begin{align*}
&(w^\top z(x)+b)^2-\eta^2w^\top S(x)w \\
&\quad = b^2
+2b\sum_{i=1}^d w_i z_i(x)
+\sum_{i=1}^d w_i^2 z_i(x)^2
+2\sum_{1\leq i<j\leq d}w_iw_j z_i(x)z_j(x) \\
&\qquad
-\eta^2\sum_{i=1}^{d_M} w_i^2 S_{ii}(x)
-2\eta^2\sum_{1\leq i<j\leq d_M}w_iw_j S_{ij}(x).
\end{align*}
Therefore there is a coefficient vector
$a_{\theta,\eta}\in\mathbb R^m$, depending only on $(w,b,\eta)$, such that
\[
(w^\top z(x)+b)^2-\eta^2w^\top S(x)w
= a_{\theta,\eta}^\top\Psi(x).
\]
Equivalently, \eqref{eq:proof-vc-quadratic} is precisely
\[
a_{\theta,\eta}^\top\Psi(x)<0.
\]
Thus each attackability indicator in $\mathcal G_\Theta$ is the pullback, by
$\Psi$, of a homogeneous linear threshold in $\mathbb R^m$.

It remains only to count the dimension of the lifted feature space. The
feature map $\Psi$ has one constant coordinate, $d_M$ linear embedding
coordinates,
\[
\frac{d_M(d_M+1)}2
\]
distinct quadratic products $z_i(x)z_j(x)$, and
\[
\frac{d_M(d_M+1)}2
\]
distinct entries of the symmetric matrix $S(x)$. Hence
\[
m=1+d_M+\frac{d_M(d_M+1)}2+\frac{d_M(d_M+1)}2=(d_M+1)^2.
\]
Each indicator $G_{\theta,\eta}\in\mathcal G_\Theta$ takes the form
$\mathbbm{1}\{a_{\theta,\eta}^\top\Psi(x)<0\}$ for some
$a_{\theta,\eta}\in\mathbb R^m$ depending on $(\theta,\eta)$. Therefore
$\mathcal G_\Theta$ is contained in the class of homogeneous linear threshold
sets on the lifted feature $\Psi(x)$. Since VC dimension is monotone under
set inclusion, $\VC(\mathcal G_\Theta)$ is at most that of the ambient class
of homogeneous halfspaces in $\mathbb R^m$, which is at most $m$ (one less
than the bound $m+1$ for inhomogeneous halfspaces). Therefore
\[
\VC(\mathcal G_\Theta)\leq m=(d_M+1)^2.
\]

\subsection{Proof of Proposition~\ref{prop:vc-improved}}

By the squaring in the proof of Proposition~\ref{prop:vc-affine}
(equation~\eqref{eq:proof-vc-quadratic}), the attackability event is
equivalent to $(w^\top z(x)+b)^2-\eta^2 w^\top S(x)w<0$. Set $\lambda=\eta^2$;
the event becomes $\lambda\,w^\top S(x)w-(w^\top z(x)+b)^2>0$, so
$\mathcal G_\Theta$ is contained in the class indexed by
$(w,b,\lambda)\in\mathbb R^{d_M+2}$ defined by this single inequality, and VC
dimension is monotone under inclusion. Represent an instance by
\[
T(x)=\bigl(z(x),\bigl(S_{ij}(x)\bigr)_{1\leq i\leq j\leq d_M}\bigr).
\]
The membership condition is a single polynomial inequality. Expanding,
\[
\lambda\,w^\top S(x)w-(w^\top z(x)+b)^2
=\sum_{i,j}\lambda\,w_iw_jS_{ij}(x)
-\sum_{i,j}w_iw_jz_i(x)z_j(x)-2b\sum_i w_iz_i(x)-b^2,
\]
and, counting the concept variables $(w,b,\lambda)$ and the entries of $T(x)$
together, the monomials $\lambda\,w_iw_jS_{ij}(x)$ and $w_iw_jz_i(x)z_j(x)$
each have degree $4$, while $b\,w_iz_i(x)$ has degree $3$ and $b^2$ degree
$2$; the total degree is therefore $4$. By
\citet[Theorem~2.2]{goldberg1995bounding},
\[
\VC(\mathcal G_\Theta)
\leq 2(d_M+2)\log_2(8e\cdot 4\cdot 1)
=
2(d_M+2)\log_2(32e).
\]
This bound is set by the number of parameters $d_M+2$, the single predicate,
and the degree $4$, not by the number of entries of $T(x)$, so
$\VC(\mathcal G_\Theta)=O(d_M)$.

\subsection{Proof of Corollary~\ref{cor:affine-vc-uniform}}

Combine Theorem~\ref{thm:vc-uniform} with Proposition~\ref{prop:vc-improved}.
The proposition supplies the bound $\VC(\mathcal G_\Theta)=O(d_M)$;
substituting this into \eqref{eq:vc-uniform-bound} and absorbing numerical
constants into $C$ gives \eqref{eq:affine-vc-uniform}.

\subsection{Proof of Theorem~\ref{thm:fat-margin}}

\begin{proof}
The proof uses a covering-number notation different from the one in
Proposition~\ref{prop:cover-margin}. There, $\mathcal N_\infty(\epsilon,\mathcal F)$
denoted the full-domain supremum-norm covering number, defined through the
distance $d_\infty(f,g)=\sup_{z\in\mathcal Z}|f(z)-g(z)|$ over the whole
domain. Theorems~10.4 and~12.13 of \citet{anthony2009neural} are instead
stated in terms of an \emph{empirical} supremum-norm covering number, which
we now define. For a class $\mathcal F$ of real-valued functions on a domain
$\mathcal Z$, a positive integer $N$, and $\epsilon>0$, the empirical
supremum-norm covering number is
\[
\mathcal N_\infty(\epsilon,\mathcal F,N)
:=\sup_{(z_1,\ldots,z_N)\in\mathcal Z^N}
\mathcal N(\epsilon,\mathcal F,(z_1,\ldots,z_N)),
\]
where $\mathcal N(\epsilon,\mathcal F,(z_1,\ldots,z_N))$ is the smallest
cardinality of a finite set $\mathcal C\subseteq\mathcal F$ such that for
every $f\in\mathcal F$ there is a $g\in\mathcal C$ with
$\max_{i\leq N}|f(z_i)-g(z_i)|\leq\epsilon$. The empirical version is
bounded above by the full-domain version,
$\mathcal N_\infty(\epsilon,\mathcal F,N)\leq\mathcal N_\infty(\epsilon,\mathcal F)$,
but is in general much smaller. The third argument $N$ distinguishes the
empirical version from the full-domain version $\mathcal N_\infty(\epsilon,\mathcal F)$
of Section~\ref{sec:margins}.

The proof proceeds in three steps. First, apply the margin-covering theorem
of \citet[Theorem~10.4]{anthony2009neural} to the class $\mathcal M_\eta$
at scale $\rho$. For every $\epsilon>0$,
\[
\mathbb P\left\{\exists\theta\in\Theta:
R_\eta(\theta)>\widehat R_{n,\eta,\rho}(\theta)+\epsilon\right\}
\leq
2\mathcal N_\infty(\rho/2,\pi_\rho(\mathcal M_\eta),2n)
\exp\left(-\frac{\epsilon^2 n}{8}\right).
\]
Here $\pi_\rho$ is the clipping map of \citet[Section~10.4]{anthony2009neural},
which restricts function values to a $\rho$-neighbourhood of the decision
threshold (zero in our margin convention); the empirical covering number on
the right-hand side is taken over samples of size $2n$, as required by the
symmetrisation argument in the proof of Theorem~10.4.

Second, the fat-shattering covering estimate of
\citet[Theorem~12.13]{anthony2009neural} bounds this empirical covering
number in terms of the fat-shattering dimension of $\pi_\rho(\mathcal M_\eta)$
at scale $\rho/8$. Since clipping does not increase the fat-shattering
dimension, $\fat_{\rho/8}(\pi_\rho(\mathcal M_\eta))\leq\fat_{\rho/8}(\mathcal M_\eta)=d_\rho$,
and Theorem~12.13 gives
\[
\log\mathcal N_\infty(\rho/2,\pi_\rho(\mathcal M_\eta),2n)
\leq Cd_\rho\log^2(CMn/\rho)
\]
for a universal constant $C$, after absorbing numerical constants.

Third, set
\[
\epsilon=\sqrt{\frac{8}{n}\left(\log\mathcal N_\infty(\rho/2,\pi_\rho(\mathcal M_\eta),2n)+\log(2/\delta)\right)}.
\]
This choice makes the displayed probability at most $\delta$. Substituting
the covering estimate from the second step into this expression for
$\epsilon$ and absorbing universal numerical constants into $C$ gives, with
probability at least $1-\delta$, uniformly over $\theta\in\Theta$,
\[
R_\eta(\theta)
\leq
\widehat R_{n,\eta,\rho}(\theta)
+C\sqrt{\frac{d_\rho\log^2(CMn/\rho)+\log(1/\delta)}{n}},
\]
which is \eqref{eq:fat-margin-bound}.
\end{proof}

\subsection{Proof of Proposition~\ref{prop:cover-margin}}

Fix two readouts $\theta=(w,b)$ and $\theta'=(w',b')$. For any labelled point
$(x,y)$, the difference between the adjusted margins is bounded by
\begin{align*}
&\left|m_{\eta,\theta}(x,y)-m_{\eta,\theta'}(x,y)\right|\\
&\qquad\leq
\|z(x)\|_2\|w-w'\|_2+|b-b'|
+\eta\left|\sqrt{w^\top S(x)w}-\sqrt{w'^\top S(x)w'}\right|.
\end{align*}
The first two terms are the ordinary Lipschitz bound for affine scores. For
the adversarial displacement term, use $S(x)\succeq0$ and
$\|S(x)\|_{\mathrm{op}}\leq\Lambda$. Then
\[
\sqrt{w^\top S(x)w}=\|S(x)^{1/2}w\|_2,
\]
and by the reverse triangle inequality,
\begin{align*}
&\left|\sqrt{w^\top S(x)w}-\sqrt{w'^\top S(x)w'}\right| \\
&\quad=\left|\|S(x)^{1/2}w\|_2-\|S(x)^{1/2}w'\|_2\right| \\
&\quad\leq \|S(x)^{1/2}(w-w')\|_2
\leq \sqrt\Lambda\|w-w'\|_2.
\end{align*}
Using $\|z(x)\|_2\leq R$, we obtain the uniform Lipschitz bound
\[
\left|m_{\eta,\theta}(x,y)-m_{\eta,\theta'}(x,y)\right|
\leq (R+\eta\sqrt\Lambda)\|w-w'\|_2+|b-b'|.
\]
Set $L:=R+\eta\sqrt\Lambda$. If the $w$-parameters are covered at Euclidean
radius $\epsilon/(2L)$ and the bias interval is covered at radius
$\epsilon/2$, then the corresponding adjusted-margin functions are covered in
supremum norm at radius $\epsilon$.

The Euclidean ball $\{w:\|w\|_2\leq W\}\subset\mathbb R^{d_M}$ has an
$\epsilon/(2L)$-net of size at most
\[
\left(1+\frac{4WL}{\epsilon}\right)^{d_M},
\]
by the standard volumetric covering bound. The interval $[-B_0,B_0]$ has an
$\epsilon/2$-net of size at most
\[
1+\frac{4B_0}{\epsilon}.
\]
Multiplying the two covering numbers gives \eqref{eq:cover-margin}.
Finally, for any $\theta\in\Theta_{W,B_0}$,
\[
|m_{\eta,\theta}(x,y)|
\leq |w^\top z(x)|+|b|+\eta\sqrt{w^\top S(x)w}
\leq WR+B_0+\eta W\sqrt\Lambda=M_\eta,
\]
which proves \eqref{eq:M-bound}.

\subsection{Proof of Corollary~\ref{cor:norm-margin}}

We apply the margin-covering theorem of
\citet[Theorem~10.4]{anthony2009neural} to the adjusted-margin class
$\mathcal M_\eta$ at scale $\rho$, exactly as in the proof of
Theorem~\ref{thm:fat-margin}. The bound there is stated in terms of the
empirical supremum-norm covering number
$\mathcal N_\infty(\rho/2,\pi_\rho(\mathcal M_\eta),2n)$, which is bounded
above by the full-domain covering number $\mathcal N_\infty(\rho/2,\mathcal M_\eta)$
of Proposition~\ref{prop:cover-margin} (clipping a cover of $\mathcal M_\eta$
gives a cover of $\pi_\rho(\mathcal M_\eta)$, and the empirical version is
dominated by the full-domain version). Substituting the covering bound of
Proposition~\ref{prop:cover-margin} at scale $\rho/2$ into Theorem~10.4 and
calibrating as in the proof of Theorem~\ref{thm:fat-margin} gives, with
probability at least $1-\delta$, uniformly over $\theta\in\Theta_{W,B_0}$,
\[
R_\eta(\theta)
\leq \widehat R_{n,\eta,\rho}(\theta)
+C\sqrt{\frac{\log\mathcal N_\infty(\rho/2,\mathcal M_\eta)+\log(1/\delta)}{n}}.
\]
Proposition~\ref{prop:cover-margin} gives
\[
\log\mathcal N_\infty(\rho/2,\mathcal M_\eta)
\leq d_M\log\left(1+\frac{8W(R+\eta\sqrt\Lambda)}{\rho}\right)
+\log\left(1+\frac{8B_0}{\rho}\right),
\]
and absorbing numerical constants into the universal constant $C$ yields
\eqref{eq:norm-margin-bound}.

\subsection{Proof of Theorem~\ref{thm:rademacher-margin}}

Let
\[
\mathcal L:=\{(x,y)\mapsto y(w^\top z(x)+b):(w,b)\in\Theta_{W,B_0}\}
\]
and
\[
\mathcal Q:=\{x\mapsto \sqrt{w^\top S(x)w}:\|w\|_2\leq W\}.
\]
By subadditivity of empirical Rademacher complexity,
\[
\widehat{\mathfrak R}_n(\mathcal M_\eta)
\leq
\widehat{\mathfrak R}_n(\mathcal L)
+\eta\widehat{\mathfrak R}_n(\mathcal Q).
\]
For the linear part, write \(\tilde w=(w,b)\) and
\(\tilde z_i=(Y_i z(X_i),Y_i)\). Since
\(\|\tilde w\|_2\leq\sqrt{W^2+B_0^2}\),
\begin{align*}
\widehat{\mathfrak R}_n(\mathcal L)
&=\mathbb E_\sigma\left[
\sup_{(w,b)\in\Theta_{W,B_0}}
\frac1n\sum_{i=1}^n\sigma_i\tilde w^\top\tilde z_i
\,\middle|\,\mathcal S_n\right] \\
&\leq
\frac{\sqrt{W^2+B_0^2}}{n}
\mathbb E_\sigma\left\|\sum_{i=1}^n\sigma_i\tilde z_i\right\|_2 \\
&\leq
\frac{\sqrt{W^2+B_0^2}}{n}
\left(\sum_{i=1}^n\|\tilde z_i\|_2^2\right)^{1/2} \\
&=
\frac{\sqrt{W^2+B_0^2}}{n}
\left(\sum_{i=1}^n(\|z(X_i)\|_2^2+1)\right)^{1/2}.
\end{align*}
For the sensitivity part, write $A_i=S(X_i)^{1/2}$, so that
$\sqrt{w^\top S(X_i)w}=\|A_iw\|_2$ for each $i$. Writing
$h_w(x)=S(x)^{1/2}w\in\mathbb R^d$ for the vector-valued map indexed by $w$,
the displacement function $w\mapsto\sqrt{w^\top S(X_i)w}$ is the composition
of the $1$-Lipschitz scalar map $u\mapsto\|u\|_2$ with the vector-valued
map $h_w(X_i)=A_iw$. The Hilbert-space vector contraction inequality for
Rademacher complexities \citep{maurer2016vector}, applied to this
composition, gives
\[
\widehat{\mathfrak R}_n(\mathcal Q)
\leq
\frac{C}{n}\,
\mathbb E_g\sup_{\|w\|_2\leq W}
\sum_{i=1}^n g_i^\top A_iw,
\]
where the \(g_i\)'s are independent standard Gaussian vectors in the target
space and \(C\) is universal. Therefore
\begin{align*}
\widehat{\mathfrak R}_n(\mathcal Q)
&\leq
\frac{CW}{n}\,
\mathbb E_g\left\|\sum_{i=1}^nA_i^\top g_i\right\|_2 \\
&\leq
\frac{CW}{n}
\left(\mathbb E_g\left\|\sum_{i=1}^nA_i^\top g_i\right\|_2^2\right)^{1/2} \\
&=\frac{CW}{n}
\left(\sum_{i=1}^n\operatorname{tr}(A_i^\top A_i)\right)^{1/2} \\
&=
CW\sqrt{\frac{\overline T_n}{n}}.
\end{align*}
Combining the two estimates proves \eqref{eq:rad-complexity-bound}.

It remains to pass from a bound on the Rademacher complexity of
$\mathcal M_\eta$ to a bound on the robust risk $R_\eta(\theta)$. Three
ingredients are needed.

First, the ramp loss. Let $\varphi_\rho:\mathbb R\to[0,1]$ be defined by
\[
\varphi_\rho(t)=
\begin{cases}
1, & t\leq 0,\\
1-t/\rho, & 0<t<\rho,\\
0, & t\geq \rho,
\end{cases}
\]
which is $1/\rho$-Lipschitz and takes values in $[0,1]$, and satisfies the
two-sided inequality
\[
\mathbbm{1}\{t\leq 0\}\leq\varphi_\rho(t)\leq\mathbbm{1}\{t\leq\rho\}.
\]
Let $\mathcal G=\varphi_\rho\circ\mathcal M_\eta$ be the ramp-loss class.

Second, the standard empirical Rademacher generalization bound for bounded
losses \citep{bartlett2002rademacher} gives, with probability at least
$1-\delta$, uniformly over $g\in\mathcal G$,
\[
\mathbb E g(X,Y)\leq \frac1n\sum_{i=1}^n g(X_i,Y_i)
+2\widehat{\mathfrak R}_n(\mathcal G)+C\sqrt{\frac{\log(1/\delta)}{n}}.
\]
Specialising to $g=\varphi_\rho\circ m_{\eta,\theta}$ for each
$\theta\in\Theta$ gives, uniformly over $\theta$ with the same probability,
\[
\mathbb E\varphi_\rho(m_{\eta,\theta}(X,Y))
\leq \frac1n\sum_{i=1}^n\varphi_\rho(m_{\eta,\theta}(X_i,Y_i))
+2\widehat{\mathfrak R}_n(\mathcal G)+C\sqrt{\frac{\log(1/\delta)}{n}}.
\]

Third, the Ledoux--Talagrand contraction inequality for Rademacher averages,
applied to the $1/\rho$-Lipschitz function $\varphi_\rho$ after centering
(the constant $\varphi_\rho(0)$ does not affect the Rademacher average),
gives
\[
\widehat{\mathfrak R}_n(\mathcal G)=
\widehat{\mathfrak R}_n(\varphi_\rho\circ\mathcal M_\eta)
\leq \frac{1}{\rho}\widehat{\mathfrak R}_n(\mathcal M_\eta).
\]
Substituting the bound on $\widehat{\mathfrak R}_n(\mathcal M_\eta)$ from
\eqref{eq:rad-complexity-bound} into the previous equation and absorbing
numerical constants into $C$ gives, uniformly over $\theta\in\Theta_{W,B_0}$,
\[
\mathbb E\varphi_\rho(m_{\eta,\theta}(X,Y))
\leq \frac1n\sum_{i=1}^n\varphi_\rho(m_{\eta,\theta}(X_i,Y_i))
+\frac{2}{\rho}\widehat{\mathfrak R}_n(\mathcal M_\eta)
+C\sqrt{\frac{\log(1/\delta)}{n}}.
\]
Finally, the two-sided ramp-loss inequality converts the expected ramp loss
to the population robust risk and the empirical ramp loss to the empirical
$\rho$-margin error:
\[
R_\eta(\theta)
=\mathbb P\{m_{\eta,\theta}(X,Y)\leq 0\}
\leq \mathbb E\varphi_\rho(m_{\eta,\theta}(X,Y)),
\]
\[
\frac1n\sum_{i=1}^n\varphi_\rho(m_{\eta,\theta}(X_i,Y_i))
\leq \frac1n\sum_{i=1}^n\mathbbm{1}\{m_{\eta,\theta}(X_i,Y_i)\leq\rho\}
=\widehat R_{n,\eta,\rho}(\theta).
\]
Combining these with the previous display gives
\eqref{eq:rademacher-margin-bound}.

The specialisation in equation~\eqref{eq:rademacher-margin-bound-deterministic}
follows from substituting the uniform bounds $\|z(x)\|_2\leq R$ and
$\operatorname{tr}S(x)\leq T$:
\[
\frac1n\sum_{i=1}^n(\|z(X_i)\|_2^2+1)\leq R^2+1,
\qquad
\overline T_n\leq T.
\]
For the rank-$q$ case, $S(x)\succeq0$ implies
$\operatorname{tr}S(x)\leq \operatorname{rank}(S(x))\|S(x)\|_{\mathrm{op}}$,
so $\operatorname{rank}S(x)\leq q$ and $\|S(x)\|_{\mathrm{op}}\leq\Lambda$
together give $T\leq q\Lambda$.

\subsection{Proof of Corollary~\ref{cor:quantile-margin}}

Define the event
\[
E_\beta:=\{(x,y):\|S(x)\|_{\mathrm{op}}\leq\Lambda_\beta\}.
\]
By the choice of $\Lambda_\beta$, $\mathbb P(E_\beta^c)\leq\beta$.

The population robust risk decomposes as
\[
R_\eta(\theta)
=\mathbb P\{m_{\eta,\theta}(X,Y)\leq 0\}
=\mathbb P\{m_{\eta,\theta}(X,Y)\leq 0,\ (X,Y)\in E_\beta\}
+\mathbb P\{m_{\eta,\theta}(X,Y)\leq 0,\ (X,Y)\in E_\beta^c\}.
\]
The second joint probability is bounded by $\mathbb P(E_\beta^c)\leq\beta$,
so
\[
R_\eta(\theta)
\leq \mathbb P\{m_{\eta,\theta}(X,Y)\leq 0,\ (X,Y)\in E_\beta\}+\beta.
\]

It remains to bound the first joint probability. The indicator
$\mathbbm{1}\{m_{\eta,\theta}(X,Y)\leq 0,\ (X,Y)\in E_\beta\}$ is a function of
$(X,Y)$ with values in $\{0,1\}$, and on the event $E_\beta$ the local
sensitivity matrix satisfies $\|S(X)\|_{\mathrm{op}}\leq\Lambda_\beta$ by
definition. The event $E_\beta$ is defined independently of the readout
$\theta$, so restricting the empirical-process argument to the subset
$\{(X,Y)\in E_\beta\}$ does not enlarge the covering number of the
margin class; equivalently, one may replace each $m_{\eta,\theta}$ by
the modified function $m_{\eta,\theta}\cdot\mathbbm 1_{E_\beta}+M\cdot
\mathbbm 1_{E_\beta^c}$, for $M$ a constant exceeding the supremum of
$m_{\eta,\theta}$ on $E_\beta^c$, and the covering number of the modified
class at any scale is bounded by the covering number of the original
class restricted to $E_\beta$, on which $\|S(X)\|_{\mathrm{op}}\leq
\Lambda_\beta$ holds uniformly. The covering-number margin bound used in
the proof of Corollary~\ref{cor:norm-margin} therefore applies with
$\Lambda$ replaced by $\Lambda_\beta$. This gives, uniformly over
$\theta\in\Theta_{W,B_0}$ with probability at least $1-\delta$,
\[
\mathbb P\{m_{\eta,\theta}(X,Y)\leq 0,\ (X,Y)\in E_\beta\}
\leq
\frac1n\sum_{i=1}^n
\mathbbm{1}\{m_{\eta,\theta}(X_i,Y_i)\leq\rho,\ (X_i,Y_i)\in E_\beta\}
+\Delta_{n,\rho,\beta},
\]
where
\[
\Delta_{n,\rho,\beta}
:=C\sqrt{\frac{
 d\log\left(1+\frac{C W(R+\eta\sqrt{\Lambda_\beta})}{\rho}\right)
 +\log\left(1+\frac{C B_0}{\rho}\right)
 +\log(1/\delta)}{n}}.
\]

The empirical term with the additional condition $(X_i,Y_i)\in E_\beta$ is
no larger than the empirical $\rho$-margin error:
\[
\frac1n\sum_{i=1}^n\mathbbm{1}\{m_{\eta,\theta}(X_i,Y_i)\leq\rho,\ (X_i,Y_i)\in E_\beta\}
\leq
\frac1n\sum_{i=1}^n\mathbbm{1}\{m_{\eta,\theta}(X_i,Y_i)\leq\rho\}
=\widehat R_{n,\eta,\rho}(\theta).
\]

Combining the three preceding inequalities gives
\[
R_\eta(\theta)
\leq \widehat R_{n,\eta,\rho}(\theta)+\Delta_{n,\rho,\beta}+\beta
\]
uniformly over $\theta\in\Theta_{W,B_0}$ with probability at least $1-\delta$,
which is \eqref{eq:quantile-margin-bound}.

  \subsection{Proof of Proposition~\ref{prop:finite-search-cover}}

If $\Sigma_w(x)=0$, both sides of the desired bound are zero and the claim
is immediate. Assume henceforth that $\Sigma_w(x)>0$.

The upper bound is immediate from inclusion. The candidates contributing to
$D^{\mathrm{lin}}_{K,\eta}(x;w,b)$ satisfy $\|u_k\|_{B(x)}\leq\eta$, i.e.,
$u_k\in\mathcal B_x(\eta)$, and the linearised worst-case readout
displacement over $\mathcal B_x(\eta)$ is $\eta\sqrt{\Sigma_w(x)}$ by the
closed-form attacker's solution of Theorem~\ref{thm:flip}, so
$D^{\mathrm{lin}}_{K,\eta}(x;w,b)\leq\eta\sqrt{\Sigma_w(x)}$.

For the lower bound, set
\[
u^\star:=-\sigma_x(\eta-\delta)\,B(x)^{-1}J_M(x)^\top w/\sqrt{\Sigma_w(x)}.
\]
The $B(x)$-norm of $u^\star$ satisfies 
\[
\|u^\star\|_{B(x)}^2
=u^{\star\top}B(x)u^\star
=\frac{(\eta-\delta)^2}{\Sigma_w(x)}\,
w^\top J_M(x)B(x)^{-1}B(x)B(x)^{-1}J_M(x)^\top w
=(\eta-\delta)^2,
\]
using $\sigma_x^2=1$, the symmetry of $B(x)^{-1}$, and
$\Sigma_w(x)=w^\top J_M(x)B(x)^{-1}J_M(x)^\top w$. Hence
$\|u^\star\|_{B(x)}=\eta-\delta$ and $u^\star\in\mathcal B_x(\eta-\delta)$.
A parallel computation, using the same identity for $\Sigma_w(x)$, gives
$-\sigma_x w^\top J_M(x)u^\star=(\eta-\delta)\sqrt{\Sigma_w(x)}$.

By the covering assumption, there exists $u_k\in\{u_1,\ldots,u_K\}$ with
$\|u_k-u^\star\|_{B(x)}\leq\delta$. The triangle inequality gives
$\|u_k\|_{B(x)}\leq\|u^\star\|_{B(x)}+\|u_k-u^\star\|_{B(x)}\leq(\eta-\delta)+\delta=\eta$,
so $u_k\in\mathcal B_x(\eta)$ and $u_k$ contributes to
$D^{\mathrm{lin}}_{K,\eta}(x;w,b)$.

The readout displacement at $u_k$ decomposes as
\[
-\sigma_x w^\top J_M(x)u_k
=(\eta-\delta)\sqrt{\Sigma_w(x)}-\sigma_x w^\top J_M(x)(u_k-u^\star).
\]

To bound the second term, we apply Cauchy--Schwarz in the $B(x)$-inner
product $\langle u,v\rangle_{B(x)}:=u^\top B(x)v$, which is a genuine inner
product because $B(x)\succ 0$, with induced norm $\|u\|_{B(x)}$. The
covering condition is stated in this norm, so the argument has to be
phrased in the same geometry. Set $v:=B(x)^{-1}J_M(x)^\top w$. Then, using
the symmetry of $B(x)^{-1}$,
\[
\langle v,u_k-u^\star\rangle_{B(x)}
=v^\top B(x)(u_k-u^\star)
=w^\top J_M(x)(u_k-u^\star),
\]
so the quantity to be bounded is exactly $|\langle v,u_k-u^\star\rangle_{B(x)}|$.
Cauchy--Schwarz in the $B(x)$-inner product gives
\[
|w^\top J_M(x)(u_k-u^\star)|
\leq \|v\|_{B(x)}\cdot\|u_k-u^\star\|_{B(x)}.
\]
The $B(x)$-norm of $v$ is
\[
\|v\|_{B(x)}^2
=v^\top B(x)v
=w^\top J_M(x)B(x)^{-1}J_M(x)^\top w
=\Sigma_w(x),
\]
again using the same identity for $\Sigma_w(x)$ as above. Combining with
$\|u_k-u^\star\|_{B(x)}\leq\delta$ from the covering hypothesis,
\[
|w^\top J_M(x)(u_k-u^\star)|
\leq \sqrt{\Sigma_w(x)}\cdot\|u_k-u^\star\|_{B(x)}
\leq \delta\sqrt{\Sigma_w(x)}.
\]

Therefore
\[
D^{\mathrm{lin}}_{K,\eta}(x;w,b)
\geq -\sigma_x w^\top J_M(x)u_k
\geq (\eta-\delta)\sqrt{\Sigma_w(x)}-\delta\sqrt{\Sigma_w(x)}
=(\eta-2\delta)\sqrt{\Sigma_w(x)}.
\]
Setting $\delta\leq\eta\epsilon$ for $\epsilon\in(0,1/2)$ gives
$D^{\mathrm{lin}}_{K,\eta}(x;w,b)\geq(1-2\epsilon)\eta\sqrt{\Sigma_w(x)}$.

\subsection{Proof of Corollary~\ref{cor:finite-search-robustness}}

By Proposition~\ref{prop:finite-search-cover}, the covering assumption gives
$D^{\mathrm{lin}}_{K,\eta}(x;w,b)\geq(1-2\epsilon)\eta\sqrt{\Sigma_w(x)}$.
The hypothesis $D^{\mathrm{lin}}_{K,\eta}(x;w,b)<\gamma_w(x)$ then
implies
\[
(1-2\epsilon)\eta\sqrt{\Sigma_w(x)}\leq D^{\mathrm{lin}}_{K,\eta}(x;w,b)<\gamma_w(x),
\]
which rearranges to $\eta\sqrt{\Sigma_w(x)}<\gamma_w(x)/(1-2\epsilon)$.

\subsection{Proof of Proposition~\ref{prop:prob-covering}}

The argument has four steps: reduce the covering condition on the
continuous ball $\mathcal B_x(\eta-\delta)$ to a covering condition on a
finite discretisation $\mathcal Z$; bound the probability that a single
random candidate $u_1$ lands near any fixed point $z_i\in\mathcal Z$;
amplify this per-point bound to the joint event over all $K$ candidates
using independence; and take a union bound over the $N$ points of
$\mathcal Z$.

\emph{Step 1: discretisation.} Let
$\mathcal Z=\{z_1,\ldots,z_N\}\subseteq\mathcal B_x(\eta-\delta)$ be a
$\delta/2$-cover of $\mathcal B_x(\eta-\delta)$ in the proxy metric, of
cardinality $N:=N(\delta/2,\mathcal B_x(\eta-\delta))$. Such a cover
exists by the definition of the covering number. The point of
discretising is that we want to reduce the continuous covering condition
to a statement about a finite set of points $z_i$, each of which we can
analyse separately.

Suppose every $z_i$ has at least one candidate within proxy-distance
$\delta/2$, that is, for every $i$ there is some index $k(i)$ with
$\|z_i-u_{k(i)}\|_{B(x)}\leq\delta/2$. Then for any
$u\in\mathcal B_x(\eta-\delta)$, choose $z_i$ with
$\|u-z_i\|_{B(x)}\leq\delta/2$ (possible because $\mathcal Z$ is a
$\delta/2$-cover) and use the candidate $u_{k(i)}$. The triangle
inequality gives
\[
\|u-u_{k(i)}\|_{B(x)}
\leq\|u-z_i\|_{B(x)}+\|z_i-u_{k(i)}\|_{B(x)}
\leq\delta/2+\delta/2=\delta,
\]
so the candidate set $\delta$-covers $\mathcal B_x(\eta-\delta)$. It
therefore suffices to upper-bound the probability of the failure event
\[
F:=\{\exists\, z_i\in\mathcal Z:\text{ no candidate lies within }\delta/2\text{ of }z_i\}.
\]

\emph{Step 2: hit probability for a single candidate.} Fix
$z_i\in\mathcal Z$ and let $\mathcal N_{\delta/2}(z_i):=\{u:\|u-z_i\|_{B(x)}\leq\delta/2\}$
be the proxy ball of radius $\delta/2$ centred at $z_i$. This ball lies
inside $\mathcal B_x(\eta)$, where the density floor applies: for any
$u\in\mathcal N_{\delta/2}(z_i)$, the triangle inequality gives
\[
\|u\|_{B(x)}
\leq\|z_i\|_{B(x)}+\|u-z_i\|_{B(x)}
\leq(\eta-\delta)+\delta/2<\eta.
\]
The probability that a single candidate $u_1$ lands in this ball is
therefore
\[
\mathbb P\{u_1\in\mathcal N_{\delta/2}(z_i)\}
=\int_{\mathcal N_{\delta/2}(z_i)} f(u)\,du
\geq c\cdot\mathrm{vol}\bigl(\mathcal N_{\delta/2}(z_i)\bigr),
\]
using $f\geq c$ throughout the integration domain.

The Lebesgue volume of $\mathcal N_{\delta/2}(z_i)$ is
$v_q(\delta/2)^q/\sqrt{\det B(x)}$. To see this, write
$\mathcal N_{\delta/2}(z_i)=z_i+B(x)^{-1/2}\bigl(v_q\text{-ball of radius }\delta/2\bigr)$:
the proxy ball at $z_i$ is the image of the Euclidean ball of radius
$\delta/2$ under the linear map $B(x)^{-1/2}$, translated by $z_i$.
Translations preserve Lebesgue volume and $B(x)^{-1/2}$ scales it by
$|\det B(x)^{-1/2}|=1/\sqrt{\det B(x)}$. Combining,
\begin{equation}\label{eq:hit-prob}
\mathbb P\{u_1\in\mathcal N_{\delta/2}(z_i)\}
\geq c\cdot v_q(\delta/2)^q/\sqrt{\det B(x)}
=\kappa\, v_q(\delta/2)^q,
\end{equation}
where $\kappa=c/\sqrt{\det B(x)}$.

\emph{Step 3: independence across candidates.} The events
$\{u_k\in\mathcal N_{\delta/2}(z_i)\}_{k=1}^K$ are independent (the $u_k$
are i.i.d.), and each has marginal probability at least
$\kappa v_q(\delta/2)^q$ by~\eqref{eq:hit-prob}. The probability that
none of them lands in $\mathcal N_{\delta/2}(z_i)$ is therefore
\[
\mathbb P\{\text{no candidate lies within }\delta/2\text{ of }z_i\}
=\prod_{k=1}^K \mathbb P\{u_k\notin\mathcal N_{\delta/2}(z_i)\}
\leq\bigl(1-\kappa v_q(\delta/2)^q\bigr)^K.
\]

\emph{Step 4: union bound over discretisation points.} The failure
event $F$ is the union of the $N$ events from Step 3, one per
$z_i\in\mathcal Z$. The union bound gives
\[
\mathbb P(F)
\leq\sum_{i=1}^N \mathbb P\{\text{no candidate within }\delta/2\text{ of }z_i\}
\leq N\bigl(1-\kappa v_q(\delta/2)^q\bigr)^K.
\]
Taking complements and applying the reduction of Step 1 yields the
stated inequality.

\emph{The covering-number rate.} The covering number
$N(\delta/2,\mathcal B_x(\eta-\delta))$ in the proxy metric is of order
$((\eta-\delta)/(\delta/2))^q=O((\eta/\delta)^q)$ for $\delta\leq\eta/2$,
by the standard volumetric bound (the ball of radius $\eta-\delta$
contains at most $((\eta-\delta)/(\delta/2))^q$ disjoint balls of radius
$\delta/4$, each of volume comparable to the candidate ball). Setting
$K=\Theta(\delta^{-q}\log N)$ then drives the failure probability below
any prescribed level, recovering the $K=\Theta(\delta^{-q}\log(1/\delta))$
rate quoted after the proposition.

\subsection{Proof of Proposition~\ref{prop:gradient-search}}

The direction $r_w(x)=-\sigma_x B(x)^{-1}J_M(x)^\top w/\sqrt{\Sigma_w(x)}$
satisfies
\[
\|r_w(x)\|_{B(x)}^2
=\frac{w^\top J_M(x)B(x)^{-1}J_M(x)^\top w}{\Sigma_w(x)}=1,
\]
so each candidate $\alpha_j r_w(x)$ with $\alpha_j\in[0,\eta]$ lies in
$\mathcal B_x(\eta)$ and contributes to $D^{\mathrm{lin}}_{K,\eta}(x;w,b)$.
The readout displacement at $\alpha_j r_w(x)$ is
\[
-\sigma_x w^\top J_M(x)\alpha_j r_w(x)
=\alpha_j\cdot
\frac{w^\top J_M(x)B(x)^{-1}J_M(x)^\top w}{\sqrt{\Sigma_w(x)}}
=\alpha_j\sqrt{\Sigma_w(x)}.
\]
By the covering assumption, the $\alpha_j$ $\delta'$-cover $[0,\eta]$, so
$\max_j\alpha_j\geq\eta-\delta'$. Therefore
\[
D^{\mathrm{lin}}_{K,\eta}(x;w,b)
\geq\max_j\alpha_j\sqrt{\Sigma_w(x)}
\geq(\eta-\delta')\sqrt{\Sigma_w(x)}.
\]

\end{document}